\documentclass{article}
\pdfoutput=1

\usepackage{arxiv}
\usepackage[utf8]{inputenc}
\usepackage[T1]{fontenc}
\usepackage{microtype}
\usepackage{graphicx}
\usepackage{booktabs}
\usepackage{array}
\usepackage{tabularx}
\usepackage{multirow}
\usepackage{amsmath,amssymb,amsthm}
\usepackage{siunitx}
\usepackage{xcolor}
\usepackage{placeins}
\usepackage{float}
\usepackage{xspace}
\usepackage[hidelinks]{hyperref}
\usepackage[nameinlink,noabbrev]{cleveref}

\hypersetup{
  pdftitle={SafeRestore: Detector-Relative Risk Certificates for Selective Industrial Image Restoration},
  pdfauthor={Shaoliang Yang and Jun Wang},
  pdfkeywords={selective image restoration, exact binomial certification, industrial inspection, abstention, uncertainty quantification}
}

\graphicspath{{figures/}}
\newcommand{\method}{SafeRestore\xspace}
\newcommand{\lossrisk}{R_{\mathrm{loss}}}
\newcommand{\actrisk}{R_{\mathrm{act}}}
\newcommand{\nLoss}{n_{\mathrm{loss}}}
\newcommand{\nAct}{n_{\mathrm{act}}}
\newcommand{\eLoss}{e_{\mathrm{loss}}}
\newcommand{\eAct}{e_{\mathrm{act}}}
\newcommand{\ucbLoss}{U_{\mathrm{loss}}}
\newcommand{\ucbAct}{U_{\mathrm{act}}}
\newcommand{\ind}{\mathbf{1}}
\newtheorem{proposition}{Proposition}

\title{\method: Detector-Relative Risk Certificates for\\
Selective Industrial Image Restoration}
\author{%
Shaoliang Yang\\
Santa Clara University\\
Santa Clara, CA 95053, USA\\
\texttt{syang11@scu.edu}
\And
Jun Wang\thanks{Corresponding author: \texttt{jwang22@scu.edu}.}\\
Santa Clara University\\
Santa Clara, CA 95053, USA\\
\texttt{jwang22@scu.edu}
}
\date{}
\renewcommand{\headeright}{Preprint}
\renewcommand{\undertitle}{Preprint}
\renewcommand{\shorttitle}{SafeRestore: Detector-Relative Risk Certificates}

\begin{document}
\maketitle

\begin{abstract}
Industrial inspection pipelines often restore a measured image before a
detector acts on it, yet restoration can suppress detector-supported defect
structure or create clean-region activations.  We formulate restoration as a
selective action problem over the measured display, five restored candidates,
and review.  \method ranks candidates with action-specific fitted scores,
chooses a gate on threshold-tuning data, and evaluates the fixed gate on a
disjoint certification sample with two one-sided exact binomial bounds: one for
the positive-conditional evidence-loss incident rate and one for the
all-accepted excess-activation incident rate.  The guarantee is marginal for one
policy fixed before its certification outcomes are observed, under an
image-level i.i.d.\ working model.  In a retrospective split-sample study of
4,591 public Carinthia-S images, the protocol yields auditable risk--coverage
behavior.  The primary all-action policy passes in one of five training repetitions
(\(12.0\%\pm26.9\%\) pass-gated test coverage when failures count as zero),
whereas fixed bicubic and reduced-complexity variants pass more often.  On
reserved morphologies, evidence-loss incidence rises to 81.1--90.3\%, and
KolektorSDD lacks both detector competence and enough positive certification
images for the stated target.  The contribution is therefore an auditable,
detector-relative framework for deciding when a transformed image may be
returned automatically and when review remains necessary---not a claim that
adaptive routing outperforms simpler policies on the present evidence.
\end{abstract}


\section{Introduction}
\label{sec:introduction}

Image restoration changes the observation on which an inspection decision is
made.  A transformation that improves visual sharpness or average
reconstruction fidelity can still suppress a small defect, and a learned prior
can introduce structure that activates a downstream detector
\cite{belthangady2019microscopy,weigert2018care,kim2026hallugen}.  These two
failures have different operational consequences and are poorly summarized by
PSNR or SSIM when the relevant structure occupies a small fraction of the image.

In an inspection pipeline, restoration is therefore a \emph{decision} problem
rather than only a reconstruction problem.  For each observation, the pipeline
may retain the measured display, return one of several restored candidates, or
assign the case to review.  The useful question is not only which candidate has
the best mean quality, but whether a fixed policy has enough held-out evidence
to return any candidate automatically while keeping two distinct incident rates
below stated limits.  Task-oriented restoration addresses downstream utility
\cite{haris2021task_driven_sr,kim2024sr4ir,chen2025unirestore}; selective
prediction and risk-control methods address abstention and finite-sample
calibration \cite{geifman2017selective,bates2021rcps,angelopoulos2021ltt}.
\method joins these perspectives at the level of the returned image action.

\method is a detector-relative selective-restoration protocol.  Its action set
contains the raw display and five fixed or learned restorations.  For each
action, class-balanced logistic models produce two fitted ranking scores from
detector and consistency features.  One score ranks
\emph{evidence-loss} risk---detector recall falling below a specified floor on a
defect-positive image.  The other ranks \emph{excess-activation}
risk---clean-region detector activation exceeding a specified threshold.  The
scores select an action and order images for gating; they are ranking scores,
not calibrated probabilities.  A gate is chosen on one data role and evaluated
once on a disjoint certification role using separate one-sided exact binomial
bounds (\Cref{fig:concept}).

The risk populations matter.  Evidence loss is meaningful only on
defect-positive images, whereas excess activation is defined on the clean
region of every accepted image.  Pooling clean images into the evidence-loss
denominator would make apparent performance depend mechanically on defect
prevalence.  \method therefore records different denominators for the two
endpoints, reports positive and clean coverage separately, and supplies a
prevalence-standardized expression for excess-activation risk when the intended
operating mixture differs from the study mixture.

\begin{figure}[!htb]
  \centering
  \includegraphics[width=\linewidth]{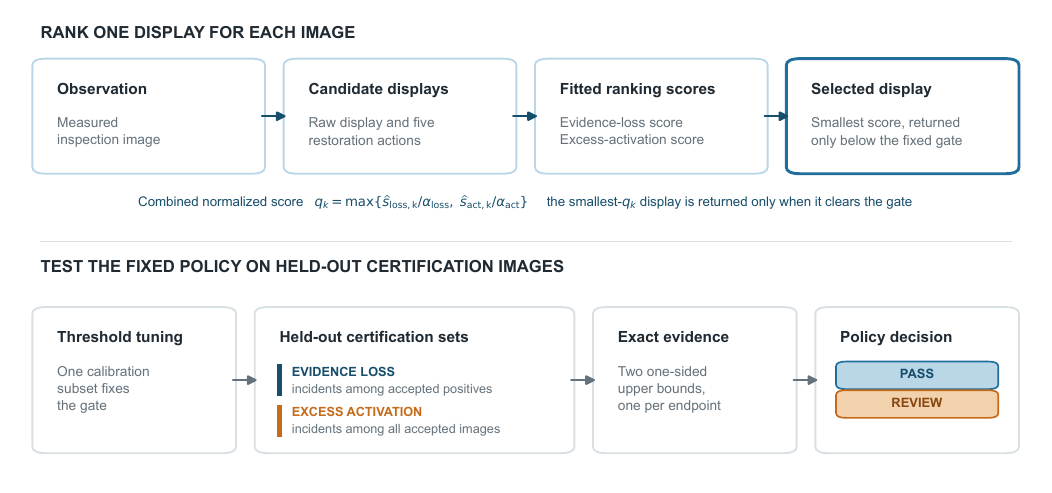}
  \caption{\method separates action ranking from policy certification.  The raw
  display and restored candidates receive two fitted ranking scores (not
  calibrated probabilities).  Their worst normalized score selects the displayed
  candidate and defines a gate.  A held-out certification sample supplies exact
  bounds for evidence loss among accepted positives and excess activation among
  all accepted images.  The policy passes only when both upper bounds are at or
  below their targets; a passing fixed policy then returns its selected action,
  and all other cases are assigned to review.}
  \label{fig:concept}
\end{figure}

The statistical statement is deliberately narrow.  For one policy fixed before
its certification outcomes are observed, an image-level i.i.d.\ working model and
a Bonferroni allocation across the two endpoints yield a finite-sample bound on
false certification.  The statement is marginal for that policy.  It does not
cover choosing whichever detector/restorer seed, action pool, or ablation
happens to pass.  The empirical study therefore reports five marginal
certificates as a sensitivity analysis and does not select a passing seed for
deployment.

The contribution is a protocol, not a new restoration network or generic
conformal theorem.  Compact restorers and detectors are used as reproducible
testbeds for the decision layer.  On a retrospective split of 4,591 public
Carinthia-S SEM image--mask pairs \cite{carinthias2025}, the primary all-action policy
passes in only one of five training repetitions, while fixed bicubic and several
simpler variants pass more often.  Certification therefore exposes training
sensitivity and withholds automatic return when evidence is insufficient.
Reserved morphologies and KolektorSDD further show where the nominal certificate
does not transfer.  Relationship to our earlier study of whether
super-resolution preserves defect evidence \cite{yang2026waferinspectsr}
is summarized once in \Cref{tab:prior-comparison}; the present estimand and
records are new, but the Carinthia-S identity inventory is shared.

\begin{samepage}
The contributions are:
\begin{enumerate}
  \item a detector-relative formulation of selective restoration with separate
  evidence-loss and excess-activation populations, component coverages, and an
  explicit prevalence-standardized activation risk;
  \item a reproducible rank--tune--certify procedure with a finite-sample
  false-certification statement for one fixed policy, an explicit boundary on
  cross-policy selection, and record-level evidence sufficient to reconstruct
  every certificate; and
  \item a five-repetition public-data study that reports risk--coverage
  behavior, compares every fixed action with adaptive pools, foregrounds the
  strongest simplifications, and diagnoses failure under stronger endpoints,
  limited certification evidence, reserved morphologies, and a second public
  dataset---and shows that the observed pass/fail pattern is unchanged under
  the examined alternative bound constructions.
\end{enumerate}
\end{samepage}

\section{Related Work}
\label{sec:related}

\subsection{Restoration evaluated through downstream evidence}

Modern restoration systems span convolutional, residual, adversarial, attention,
and transformer designs
\cite{dong2014srcnn,lim2017edsr,zhang2018rcan,ledig2017srgan,wang2021realesrgan,liang2021swinir,chen2022nafnet}.
Scientific-imaging studies and hallucination benchmarks distinguish plausible
appearance from measured structure
\cite{belthangady2019microscopy,weigert2018care,kim2026hallugen}.
Task-oriented restoration optimizes a downstream objective rather than visual
quality alone
\cite{haris2021task_driven_sr,kim2024sr4ir,chen2025unirestore}.  RL-Restore and
Path-Restore learn image-dependent restoration programs or network paths
\cite{yu2018rlrestore,yu2021pathrestore}.  These methods principally optimize
expected task performance, reconstruction quality, or computation.  \method
starts after the candidate transforms and downstream detector are trained.  Its
object is the selective policy that decides which image, if any, is returned
without review under two detector-relative incident limits.

\subsection{Uncertainty, selective prediction, and routing}

Distribution-free image-to-image uncertainty methods construct pixelwise
intervals for inverse problems \cite{angelopoulos2022image}; task-driven methods
propagate uncertainty to downstream functionals \cite{wen2024taskuq}; and recent
super-resolution work provides conformal confidence masks or reconstruction
sets \cite{adame2025sr_guarantees}; Bayesian alternatives instead estimate
spectral uncertainty for super-resolution \cite{liu2023spectral}.  Industrial applications include conformal
risk control for defect segmentation
\cite{shen2025conformal_defect} and calibrated pixelwise intervals for
layout-to-SEM reconstruction \cite{chien2025layout_sem_conformal}.  These works
quantify uncertainty around an image or task output.  The present endpoint is
instead the incident rate of a discrete raw/restored/review policy.

Selective prediction supplies the reject-option perspective
\cite{geifman2017selective,geifman2019selectivenet}.  Learning-to-defer methods
optimize accuracy--cost tradeoffs when passing decisions to an expert
\cite{mozannar2020defer,narasimhan2022defer}.  Risk-controlling prediction sets,
Learn-then-Test, and conformal risk control provide general finite-sample tools
for black-box losses
\cite{bates2021rcps,angelopoulos2021ltt,angelopoulos2024crc}.  Selective and joint
extensions address selection-conditioned or simultaneous quantities
\cite{xu2025scrc,yu2026joint,score2026}.

LEC is especially close conceptually: it formulates selection-conditioned risk
as a linear expectation constraint, computes a retention-maximizing threshold
from held-out data, and extends the construction to two-model routing
\cite{wang2026lec}.  RACER likewise combines finite-sample calibration with
model routing and abstention for set-valued large-language-model routing
\cite{hao2026racer}.  \method differs in the scientific endpoint, not in
claiming a more general theorem.  It selects among image transformations and
controls two incident rates with different conditioning populations.  The
implementation uses a conservative split: tuning chooses one gate, and an
independent certification role evaluates that fixed gate with two exact
binomial bounds and a Bonferroni allocation.  This construction is motivated by
risk-control principles but is not the RCPS or Learn-then-Test algorithm.

\subsection{Industrial inspection setting}

Public defect datasets make reproducible evaluation possible without
proprietary fabrication imagery.  MVTec AD and KolektorSDD contain clean and
defective surface images \cite{bergmann2019mvtec,tabernik2019kolektor}, whereas
Carinthia-S provides SEM image--mask pairs across six morphology classes
\cite{carinthias2025}.  Semiconductor-imaging work increasingly studies learned
enhancement, hotspot imaging, segmentation under noisy and data-limited
acquisition, and transfer under manufacturing constraints
\cite{sun2025confocal_sr,srfabnet2025,kim2024tsm_hotspot,chu2025tsm_transfer,jo2024tsm_tem}.
Our controlled transformations isolate the action decision; they are not a
physical qualification of a scanner, recipe, or fabrication process.

\subsection{Relationship to a prior Carinthia-S study}

The closest empirical neighbor is our earlier study of whether
super-resolution preserves defect evidence at a low false-call operating
point \cite{yang2026waferinspectsr}.  It shares authorship, several transformations,
and all 4,591 Carinthia-S image--mask identities.  The overlap is complete and
must be considered when judging empirical novelty.  \Cref{tab:prior-comparison}
separates the shared substrate from the new estimand and evidence.  No fitted
model, detector operating point, split, synthetic record, or numerical result
from that study enters score fitting, threshold selection, or
certification here.  Conversely, the present work does not reuse the shared
images as evidence of external-domain replication.

The two studies also reach compatible conclusions by different routes, and that
relationship should be read carefully.  The earlier study compared
\emph{fixed} reconstruction methods at a predeclared low-false-call operating
point and found that the highest-fidelity learned reconstructions recovered
fewer defect pixels than bicubic interpolation.  The present study does not
re-test that comparison.  It asks whether an \emph{image-specific} policy over
comparable transformations can be certified for automatic return, and it uses
different models---a compact U-Net detector and residual restorer here, rather
than the DeepLabV3 detector and jointly trained reconstruction model there---so
no operating point, threshold, or fidelity ranking is inherited.  This is not a
replication of the earlier ordering on new images.  Within this separate
retrospective study, allowing per-image adaptivity and abstention did not produce
more passing marginal certificates than simple fixed interpolation.

\begin{table}[!htb]
\centering
\caption{Provenance and claim boundary relative to
\cite{yang2026waferinspectsr}.  Carinthia-S identity overlap is complete;
novelty rests on the selective-policy estimand, disjoint certification role,
and corresponding records rather than on new public images.}
\label{tab:prior-comparison}
\footnotesize
\renewcommand{\arraystretch}{1.14}
\begin{tabularx}{\linewidth}{@{}>{\raggedright\arraybackslash}p{0.18\linewidth}>{\raggedright\arraybackslash}X>{\raggedright\arraybackslash}X@{}}
\toprule
Dimension & Prior study & This study \\
\midrule
Scientific question & Does each fixed super-resolution method preserve sparse
defect evidence at a low false-call operating point? & Does one fixed
image-specific raw/restored/review policy satisfy two stated incident limits on
its accepted population? \\
Primary evidence & Ten repetitions on independently generated line/space and
contact-hole structures; Carinthia-S is an unchanged-policy external stress. &
A retrospective Carinthia-S train/validation/tune/certify/test partition;
reserved morphologies and KolektorSDD are descriptive boundary checks. \\
Decision unit & A method--repetition pair is compared with other fixed methods.
& Each image receives one selected action or review from a fixed policy. \\
Statistical claim & Paired preservation and operating-point comparisons; no
image-specific selective-risk certificate. & Separate exact marginal bounds for
evidence loss among accepted positives and excess activation among all accepted
images. \\
Headline finding & Reconstruction fidelity and inspection utility diverge:
learned reconstruction attains the highest similarity yet detects fewer defect
pixels than bicubic interpolation. & In these five repetitions, per-image action selection with abstention passes
fewer marginal certificates than fixed bicubic and several simplifications;
training sensitivity and positive-evidence volume constrain the result. \\
Models & DeepLabV3 detector and a jointly trained reconstruction model. &
Compact U-Net detector and residual restorer; no operating point or checkpoint
is inherited from that study. \\
Shared material & Complete Carinthia-S identity inventory and several
transformation families. & The same identities and related transformations;
this paper does not present Carinthia-S as a new acquisition bank. \\
Distinct artifacts & Records, detectors, and operating points from that study.
& New split manifests, trained checkpoints, action-level score records,
tuned gates, certification counts, ablations, and policy outputs. \\
\bottomrule
\end{tabularx}
\end{table}

\section{Decision Problem and Detector-Relative Risks}
\label{sec:problem}

Let \(Y\in[0,1]^{H\times W}\) be a reference inspection image with binary
defect mask \(M\).  A controlled observation is
\begin{equation}
  X=\mathcal{D}_{s}(Y;\epsilon),
  \label{eq:degradation}
\end{equation}
where \(s\) indexes degradation severity.  The policy receives \(X\); the
reference and mask are used only in their designated offline data roles.
Equation~\eqref{eq:degradation} creates a matched counterfactual restoration
test, not a second physical acquisition.  One image is the statistical unit.
Because the public manifest contains no wafer, lot, or specimen-group
identifiers, physical independence cannot be verified; the certificate is
stated under an image-level i.i.d. working model.

The non-review action set is
\begin{equation}
  \mathcal{A}_0(X)=
  \{a_{\mathrm{raw}},a_{\mathrm{bil}},a_{\mathrm{bic}},
  a_{\mathrm{smooth}},a_{\mathrm{sharp}},a_{\mathrm{learned}}\}.
  \label{eq:action-set}
\end{equation}
The raw action displays \(X\) by nearest-neighbor resizing; the other actions
are fixed or learned transformations.  A detector \(f\), fixed within each
training repetition, maps every candidate to a pixelwise score map.  Review,
\(a_{\mathrm{review}}\), is the terminal action when the selected candidate is
not returned automatically.

\subsection{Image-level incidents}

At detector threshold \(\tau\), define defect-pixel recall and clean-region
false-positive rate for action \(a\) as
\begin{align}
 \operatorname{Rec}(a) &=
 \frac{\sum_u M_u\,\ind\{f(a(X))_u\ge\tau\}}
 {\sum_u M_u}, \qquad \sum_uM_u>0, \\
 \operatorname{FPR}_{c}(a) &=
 \frac{\sum_u(1-M_u)\,\ind\{f(a(X))_u\ge\tau\}}
 {\sum_u(1-M_u)}.
 \label{eq:image-metrics}
\end{align}
The two incident indicators are
\begin{align}
 \ell_{\mathrm{loss}}(a) &=
 \ind\{\operatorname{Rec}(a)<r_0\},
 &&\text{defined only when }M\ne0,\\
 \ell_{\mathrm{act}}(a) &=
 \ind\{\operatorname{FPR}_{c}(a)>f_0\},
 &&\text{defined for every image.}
 \label{eq:incident-losses}
\end{align}
The primary values are \(r_0=0.25\) and \(f_0=0.002\).  The first endpoint
marks severe loss of detector-supported evidence; it is not a claim that 25\%
recall is adequate segmentation.  A 50\% recall floor is analyzed separately.
The second endpoint marks excessive detector activation outside annotated
defects.

Both endpoints are detector-relative surrogates.  An evidence-loss incident can
reflect weak detector competence as well as suppression by a transformation,
and it can occur for the raw display.  An excess-activation incident is not
proof that a restoration invented physical structure; it records detector calls
on annotated clean pixels, including calls already present in the raw display.
Consequently, the paper uses the terms \emph{evidence loss} and \emph{excess
activation}, not causal claims of deletion or invention.

Let \(a^\star(X)\in\mathcal A_0(X)\) be the candidate selected before
certification, and let
\(G_\theta(X)=\ind\{q^\star(X)\le\theta\}\) be its fixed return gate.  The
policy risks are
\begin{align}
 \lossrisk(\theta)
 &=\Pr\!\left\{\ell_{\mathrm{loss}}(a^\star)=1
   \mid G_\theta=1,\ M\ne0\right\},\\
 \actrisk(\theta)
 &=\Pr\!\left\{\ell_{\mathrm{act}}(a^\star)=1
   \mid G_\theta=1\right\}.
 \label{eq:population-risks}
\end{align}
Conditioning evidence loss on \(M\ne0\) prevents accepted clean images from
artificially lowering that risk.  Excess activation uses all accepted images
because every image contains a clean region.

\subsection{Coverage and operating prevalence}

Define positive and clean return coverage as
\begin{align}
 C_+(\theta)&=\Pr\{G_\theta=1\mid M\ne0\},\\
 C_0(\theta)&=\Pr\{G_\theta=1\mid M=0\}.
 \label{eq:population-coverage}
\end{align}
For an operating defect prevalence \(\pi\), overall return coverage is
\begin{equation}
 C_\pi(\theta)=\pi C_+(\theta)+(1-\pi)C_0(\theta).
 \label{eq:standardized-coverage}
\end{equation}
Because the all-accepted activation risk changes with the accepted population
mixture, define the component risks
\begin{align}
 R_{\mathrm{act},+}(\theta)
 &=\Pr\{\ell_{\mathrm{act}}=1\mid G_\theta=1,M\ne0\},\\
 R_{\mathrm{act},0}(\theta)
 &=\Pr\{\ell_{\mathrm{act}}=1\mid G_\theta=1,M=0\}.
\end{align}
When \(C_\pi(\theta)>0\), the prevalence-standardized activation risk is
\begin{equation}
 R_{\mathrm{act},\pi}(\theta)=
 \frac{\pi C_+(\theta)R_{\mathrm{act},+}(\theta)
 +(1-\pi)C_0(\theta)R_{\mathrm{act},0}(\theta)}
 {C_\pi(\theta)}.
 \label{eq:standardized-activation}
\end{equation}
Equations~\eqref{eq:standardized-coverage}--\eqref{eq:standardized-activation}
separate a change in workload mixture from a change in component behavior.  A
component risk is undefined if its conditioning event has zero probability; an
empirical certificate with no accepted observations for a required endpoint is
set to fail rather than treating the missing rate as zero.

The study targets are
\begin{equation}
  \lossrisk(\theta)\le\alpha_{\mathrm{loss}}=0.15,
  \qquad
  \actrisk(\theta)\le\alpha_{\mathrm{act}}=0.15.
  \label{eq:targets}
\end{equation}
These are transparent research thresholds, not manufacturing acceptance limits.
A deployment study must elicit endpoint definitions and tolerances from the
costs of missed evidence, false detector activity, and review.  The two limits
are simultaneous constraints: improvement in one endpoint cannot compensate
for failure of the other.

\section{Selective Restoration Policy}
\label{sec:method}

\subsection{Candidate actions and common detector}

The action pool contains nearest-neighbor display of the controlled observation,
bilinear and bicubic interpolation, Gaussian-smoothed bicubic, sharpened bicubic,
and a compact learned residual restorer.  The learned model applies an input
convolution, six 32-channel residual blocks, and an output residual added to the
bicubic image.  Smoothing, sharpening, and learned priors are included because
they induce visibly different evidence-loss--activation tradeoffs, not because
they represent an exhaustive set of restoration models.  Architecture breadth is
not the scientific claim; the claim is the selective decision layer applied to
a fixed candidate pool.

A compact three-level U-Net detector \cite{ronneberger2015unet} is trained on
reference-resolution training images.  Within a repetition, the same detector
checkpoint and pixel threshold are applied to all six actions.  Candidate
comparisons therefore change the image transformation while holding the
pixel-level decision rule fixed.  The policy is post-hoc with respect to the
candidate transforms and detector, but application requires generating and
scoring all candidates; no efficiency benefit is assumed.

\subsection{Action-specific ranking scores}

For candidate \(a_k(X)\), let \(P_k=f(a_k(X))\) and let \(P_{\rm raw}\) be the
raw-action detector map.  Six application-time features form \(z_k(X)\):
\begin{enumerate}
  \item mean binary entropy of \(P_k\);
  \item mean squared error after bicubic projection of \(a_k(X)\) back to the
  observed low-resolution grid;
  \item mean absolute detector-map difference \(|P_k-P_{\rm raw}|\);
  \item absolute change in mean detector score relative to \(P_{\rm raw}\);
  \item mean absolute image residual from the raw action; and
  \item fraction of pixels for which \(P_k\ge0.5\).
\end{enumerate}
No reference image or defect mask is required to compute these features at
application time.

For each action, two standardized, class-balanced logistic regressions are fit
on the detector-validation role.  The evidence-loss regression uses only
positive images; the excess-activation regression uses all images.  Denote their
sigmoid outputs by \(\widehat s_{\mathrm{loss},k}(X)\) and
\(\widehat s_{\mathrm{act},k}(X)\).  Class balancing changes the fitted class
prior, so these outputs are treated only as ranking scores.  They are neither
claimed nor required to be calibrated probabilities.

Candidate \(k\) receives the normalized worst-endpoint score
\begin{equation}
 q_k(X)=\max\!\left\{
 \frac{\widehat s_{\mathrm{loss},k}(X)}{\alpha_{\mathrm{loss}}},
 \frac{\widehat s_{\mathrm{act},k}(X)}{\alpha_{\mathrm{act}}}
 \right\}.
 \label{eq:minimax-score}
\end{equation}
The selected action and gate score are
\begin{equation}
 a^\star(X)=\arg\min_{a_k\in\mathcal A_0(X)}q_k(X),
 \qquad q^\star(X)=\min_k q_k(X).
 \label{eq:action-selection}
\end{equation}
Ties follow the fixed order raw, bilinear, bicubic, smoothed bicubic, sharpened
bicubic, and learned.  The raw display therefore competes directly with every
restoration.  The fitted scores determine an ordering; held-out incident counts
determine whether the resulting gate passes certification.

\subsection{Threshold selection and certification}

The calibration inventory is partitioned by a deterministic, seed-independent
identifier-hash rule into approximately 30\% threshold-tuning and 70\%
certification roles.  On the tuning role, candidate thresholds are the ordered
unique values of \(q^\star\).  For each prefix, the empirical evidence-loss rate
is computed among accepted positives and the empirical excess-activation rate
among all accepted images.  A prefix is eligible only if it contains at least 15
accepted images and 15 accepted positives and satisfies
\begin{equation}
 \widehat R_{\mathrm{loss}}\le
 0.5\alpha_{\mathrm{loss}},\qquad
 \widehat R_{\mathrm{act}}\le
 0.5\alpha_{\mathrm{act}}.
 \label{eq:tuning-margin}
\end{equation}
The largest eligible prefix fixes \(\widehat\theta\).  If none is eligible, the
policy has no automatic-return region.  The half-target margin is a heuristic
for choosing a candidate gate; it has no coverage guarantee.  Statistical
validity comes only from evaluating the fixed gate on the disjoint certification
role.

\begin{table}[!htb]
\centering
\caption{Rank--tune--certify procedure for one fixed policy.}
\label{tab:algorithm}
\small
\begin{tabularx}{\linewidth}{@{}p{0.7cm}X@{}}
\toprule
Step & Operation \\
\midrule
1 & Fit action-specific evidence-loss scores on positive detector-validation
images and excess-activation scores on all detector-validation images. \\
2 & On threshold-tuning images, select \(a^\star\) and choose the largest prefix
meeting the minimum-evidence and half-target criteria. \\
3 & Fix the detector, score models, candidate order, and gate threshold without
using certification outcomes. \\
4 & On the certification role, count \((\eLoss,\nLoss)\) among accepted positives
and \((\eAct,\nAct)\) among all accepted images; the policy passes only if both
exact upper bounds meet their targets. \\
5 & For a policy fixed in advance, return \(a^\star(X)\) when its certificate
passes and the gate accepts; otherwise assign the image to review. \\
\bottomrule
\end{tabularx}
\end{table}

\Cref{tab:algorithm} separates score fitting, threshold choice, certification,
and the application decision.  The reproducibility archive stores, for each
image and candidate action, the six features, incident labels, selected action,
gate score, threshold, certification counts, exact bounds, and
population-specific coverage.  These records are sufficient to reconstruct every policy
comparison without retraining the detector or restorer.

\section{Exact Marginal Certification Guarantee}
\label{sec:guarantee}

Let \(\mathcal F_0\) denote all information used before certification, including
the trained detector and restorer, fitted score models, action order, detector
threshold, and gate threshold \(\widehat\theta\).  The guarantee concerns one
policy for which these objects are \(\mathcal F_0\)-measurable.  Conditional on
\(\mathcal F_0\), let
\(Z_i=(X_i,M_i)\), \(i=1,\ldots,N_{\mathrm{cert}}\), be i.i.d. image-level units
from the nominal distribution \(\mathcal D\), independent of the data used to
construct \(\mathcal F_0\).  This is the working sampling model.  The public
Carinthia-S metadata do not identify wafers, lots, specimens, or repeated fields
of view, so the physical independence assumption cannot be audited from the
manifest.

For the fixed policy, define
\begin{align}
 \nLoss &= \sum_{i=1}^{N_{\mathrm{cert}}}
 \ind\{G_{\widehat\theta}(X_i)=1,M_i\ne0\},\\
 \eLoss &= \sum_{i=1}^{N_{\mathrm{cert}}}
 \ind\{G_{\widehat\theta}(X_i)=1,M_i\ne0,
               \ell_{\mathrm{loss},i}=1\},\\
 \nAct &= \sum_{i=1}^{N_{\mathrm{cert}}}
 \ind\{G_{\widehat\theta}(X_i)=1\},\\
 \eAct &= \sum_{i=1}^{N_{\mathrm{cert}}}
 \ind\{G_{\widehat\theta}(X_i)=1,
               \ell_{\mathrm{act},i}=1\}.
 \label{eq:certificate-counts}
\end{align}
Conditional on \(\mathcal F_0\) and the corresponding accepted-sample size,
\(\eLoss\) is binomial with parameter \(\lossrisk(\widehat\theta)\), and
\(\eAct\) is binomial with parameter \(\actrisk(\widehat\theta)\).  The two
counts share observations and need not be independent.

For \(0<\gamma<1\), define the one-sided Clopper--Pearson upper bound
\cite{clopper1934}
\begin{equation}
 U(e,n;\gamma)=
 \begin{cases}
  1, & n=0\ \text{or}\ e=n,\\
  \operatorname{Beta}^{-1}(1-\gamma;e+1,n-e), & 0\le e<n.
 \end{cases}
 \label{eq:cp-bound}
\end{equation}
The convention \(U(e,0;\gamma)=1\) ensures that absence of endpoint evidence
cannot produce a passing certificate.  With joint error level \(\delta\), set
\begin{equation}
 \ucbLoss=U(\eLoss,\nLoss;\delta/2),\qquad
 \ucbAct=U(\eAct,\nAct;\delta/2).
 \label{eq:dual-bounds}
\end{equation}
The fixed policy passes when
\(\ucbLoss\le\alpha_{\mathrm{loss}}\) and
\(\ucbAct\le\alpha_{\mathrm{act}}\).

\begin{proposition}[False certification for one fixed policy]
\label{prop:false-certification}
Under the conditional i.i.d. and data-separation conditions above,
\begin{multline}
 \Pr\!\Bigl(
  \ucbLoss\le\alpha_{\mathrm{loss}},\
  \ucbAct\le\alpha_{\mathrm{act}},\\
  \text{and }[\lossrisk(\widehat\theta)>\alpha_{\mathrm{loss}}
  \ \text{or}\
  \actrisk(\widehat\theta)>\alpha_{\mathrm{act}}]
  \ \bigm|\ \mathcal F_0\Bigr)\le\delta.
 \label{eq:false-certification}
\end{multline}
The same inequality holds unconditionally by iterated expectation.
\end{proposition}

\begin{proof}
For either endpoint \(j\), exact one-sided Clopper--Pearson coverage gives
\(\Pr\{R_j>U_j\mid\mathcal F_0\}\le\delta/2\), including after conditioning on
the accepted-sample size.  If a policy passes while either target is false, at
least one of the two upper bounds fails to cover its incident probability.  A
union bound over the two endpoints gives \(\delta\).  Independence between the
endpoint indicators is unnecessary.
\end{proof}

\paragraph{Policy-family boundary.}
\Cref{prop:false-certification} is marginal for one policy fixed without its
certification outcomes.  It does not cover searching across training seeds,
action pools, feature sets, thresholds, or architecture families and then
choosing a passing member.  If five separately certified policies were treated
as a menu and any passing policy could be selected, the direct union bound would
be at most \(5\delta=0.50\) at \(\delta=0.10\), not 0.10.  A prospective study
must therefore fix one policy, allocate an error budget across a prespecified
family, use an appropriate familywise procedure, or collect a fresh
certification sample after selection.  In this paper, seed-level pass counts
measure model-fitting sensitivity; no passing seed is selected for deployment.

\paragraph{Sampling and transport boundary.}
The proposition is a repeated-sampling statement, not a posterior probability
that an individual returned image is safe.  It also requires the stated sampling
law; arbitrary exchangeability, correlated crops, or unrecorded wafer-level
clustering do not automatically yield the conditional binomial model.  Clustered
production data require cluster-level sampling or a dependence-aware analysis.
A certificate also does not transfer to a new morphology, scanner, recipe, or
prevalence mixture without additional assumptions and evidence.

\paragraph{Retrospective evidence boundary.}
The Carinthia-S identities also appear in \cite{yang2026waferinspectsr}
(\Cref{tab:prior-comparison}).  The present train/validation/tune/certify/test
roles and policy records are distinct, but the physical images are not a newly
acquired confirmation set.  The proposition states what the split-sample
procedure controls under its working model; the empirical evidence should be
interpreted as retrospective.  Prospective, cluster-identified acquisitions are
needed before making a production claim.

\section{Evaluation Design}
\label{sec:experiments}

The evaluation asks three questions, and \Cref{sec:results} answers them in this
order.  First, does a fixed selective policy pass both held-out risk bounds, and
how does it compare with simpler fixed actions?  Second, do the candidate actions
produce materially different fidelity and detector-incident profiles, so that
there is something for a policy to exploit?  Third, where does the result fail
when the endpoint, feature set, evidence volume, morphology, architecture, or
dataset is changed?

\subsection{Carinthia-S population and data roles}

Carinthia-S contains 4,591 public SEM image--mask pairs from six morphology
classes \cite{carinthias2025}.  Classes 3, 4, and 6 define the nominal
population.  A deterministic identifier-hash partition assigns 3,186 images to
training, 431 to detector validation, 461 to calibration, and 446 to held-out
test.  A nonempty mask defines a positive image.  The test role contains 430
positives and 16 clean images; one class-6 test image has a nonempty mask and is
counted as positive.  Classes 1, 2, and 5 form a disjoint 67-image
morphology-shift diagnostic and are excluded from fitting, threshold selection,
and nominal certification.

For each image, a fixed schedule assigns mild, moderate, or severe degradation.
The reference is blurred with \(\sigma_b\in\{0.6,1.2,2.0\}\), downsampled by a
factor of two, and corrupted with Gaussian noise
\(\sigma_n\in\{0.010,0.025,0.045\}\), respectively.  Assignment depends only on
the split role and image identifier, not on training seed or action; all actions
therefore receive the same controlled observation for a given image.  These
matched transformations test restoration behavior under known information loss.
They are not a specified SEM acquisition model (\Cref{fig:study-design}).

\begin{figure}[!htb]
  \centering
  \includegraphics[width=\linewidth]{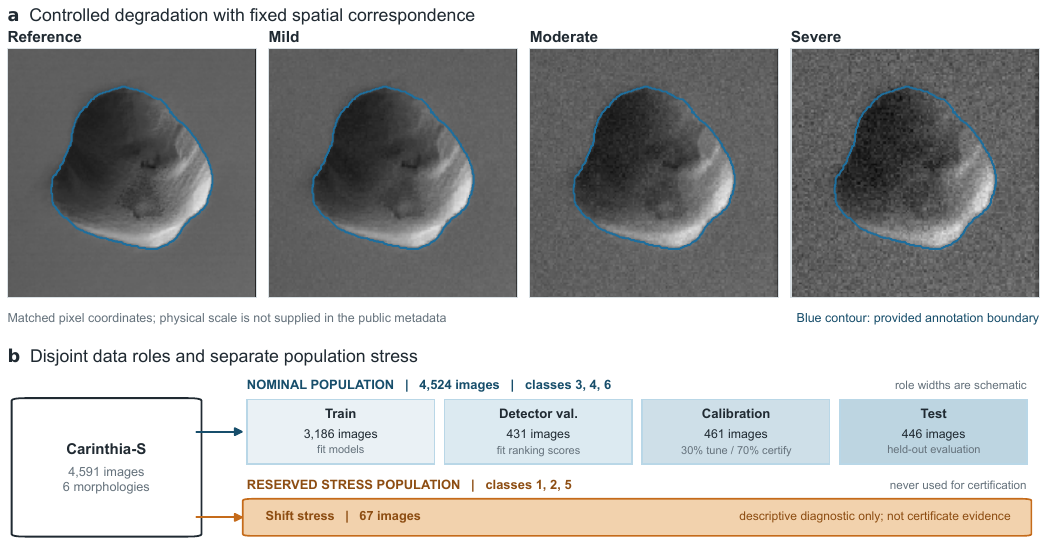}
  \caption{Controlled observations and image-level data roles.  (a) A reference
  image and the three blur--noise levels share pixel correspondence; the blue
  contour traces the provided annotation boundary.  Physical scale is unavailable in the
  public metadata.  (b) Nominal classes are partitioned into model fitting,
  score fitting, threshold tuning/certification, and held-out test roles.
  Reserved classes form a separate descriptive stress and never contribute to a
  nominal certificate.}
  \label{fig:study-design}
\end{figure}

\begin{figure}[!htb]
  \centering
  \includegraphics[width=\linewidth]{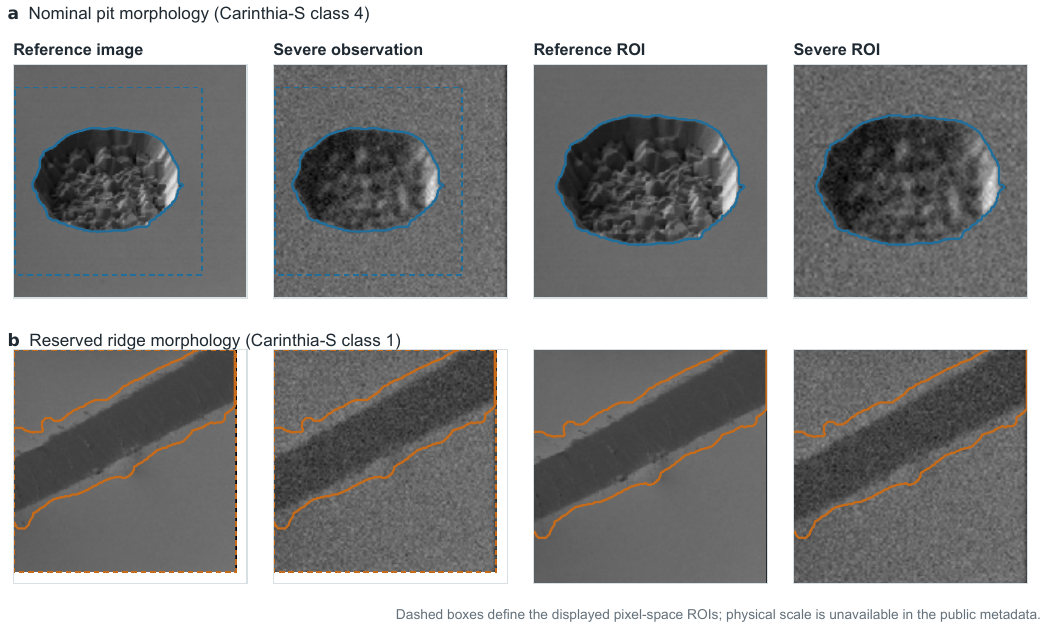}
  \caption{Image evidence for the nominal and reserved populations.  Each row
  shows a full reference, its severe controlled observation, and matching
  pixel-space regions of interest.  (a) A nominal pit morphology from the test
  population (Carinthia-S class~4).  (b) A reserved ridge morphology
  (Carinthia-S class~1).  Solid contours trace provided annotation boundaries; dashed boxes
  define the displayed regions of interest.  The examples illustrate appearance
  shift and are not themselves certification evidence.}
  \label{fig:domain-shift}
\end{figure}

\Cref{fig:domain-shift} illustrates the nominal--reserved appearance
difference; the examples are qualitative and do not enter certification.

\subsection{Models, repetitions, and endpoints}

The primary family combines a compact three-level U-Net detector with a
six-block residual restorer.  An alternate family combines an attention-gated
U-Net detector with a U-Net residual restorer under the same data roles and
endpoints.  Both families are trained from scratch at seeds 201, 203, 207, 211,
and 223.  The observation schedule and tune/certify membership remain fixed, so
these repetitions vary model fitting rather than the test sample.  Compact
architectures are deliberate: they provide reproducible checkpoints for the
decision protocol rather than a claim of state-of-the-art restoration quality.

Within each repetition, the detector pixel threshold is selected on clean
regions of the detector-validation role to target a pixel-level false-positive
rate of \(10^{-3}\).  It is then fixed for all actions.  The primary
evidence-loss incident is recall below 0.25; the boundary analysis refits the
score model and policy at a 0.50 floor.  Excess activation is clean-region FPR
above 0.002.  Both policy-risk targets are 0.15, with joint marginal
false-certification level \(\delta=0.10\).  Calibration membership is divided
approximately 30\%/70\% into tuning and certification using a deterministic
seed-independent identifier-hash rule (\Cref{tab:protocol}).

\begin{table}[!htb]
\centering
\caption{Evaluation specification and claim boundary.}
\label{tab:protocol}
\small
\begin{tabularx}{\linewidth}{@{}>{\raggedright\arraybackslash}p{3.2cm}X@{}}
\toprule
Element & Setting \\
\midrule
Statistical unit & One image under an i.i.d.\ working model; wafer/lot grouping is unavailable \\
Nominal population & Carinthia-S classes 3, 4, and 6; 461 calibration and 446 test images (430 positive, 16 clean) \\
Reserved population & Classes 1, 2, and 5; 67 positives; descriptive transfer diagnostic only \\
Repetitions & Seeds 201, 203, 207, 211, and 223; recurring identities; descriptive training sensitivity \\
Tune / certify & Deterministic identifier-hash split; one tuned threshold evaluated on the certification role \\
Primary incidents & Positive-conditional recall \(<0.25\); all-accepted clean-region FPR \(>0.002\) \\
Boundary endpoint & Positive-conditional recall \(<0.50\), with score refitting and threshold reselection \\
Risk limits & \(\alpha_{\mathrm{loss}}=\alpha_{\mathrm{act}}=0.15\), marginal joint \(\delta=0.10\) \\
Minimum tuning evidence & 15 accepted images and 15 accepted positives \\
Second dataset & KolektorSDD finite-sample and detector-competence stress; no passing-policy claim \\
\bottomrule
\end{tabularx}
\end{table}

\subsection{Comparators and outcomes}

The all-action policy, which ranks the raw display together with all five
restorations, is the primary adaptive comparator.  Fixed-action policies use
raw, bilinear, bicubic, smoothed bicubic, sharpened bicubic, or learned
restoration alone.  A restoration-only adaptive pool excludes raw.  Supporting
comparators include entropy-gated raw display and an outcome-informed oracle
that tests whether even idealized action choice can rescue an accept-all gate.
Every comparator is fit, tuned, and certified separately; none inherits the
primary policy's certificate.

The main outcome is the number of marginal seed-level certificates that pass.
\emph{Conditional coverage} is the fraction below the tuned gate before applying
the certificate decision.  \emph{Pass-gated coverage} equals conditional
coverage for a passing policy and zero otherwise.  Across seeds, its mean and
standard deviation summarize training sensitivity, not confidence intervals or
independent-factory replication.  We also report positive and clean coverage,
endpoint-specific certification counts and bounds, risk--coverage curves on the
certification role, and incident rates on the held-out test role.

Feature ablations remove one or more score inputs and pool ablations remove
candidate actions.  They are distinct policies with separate marginal
certificates.  Their purpose is to test whether the all-action
policy is actually supported over simpler alternatives.  The
study configuration was not externally preregistered.  Split manifests,
image--action records, exact counts, and figure source tables are released so
all reported comparisons can be reconstructed.

\subsection{KolektorSDD boundary study}

KolektorSDD contains 399 images from 50 items \cite{tabernik2019kolektor}.  We
preserve item-level separation.  Official fold~0 is too small for the minimum
positive tuning requirement.  A larger item-level diagnostic split contains 18
calibration positives, but its deterministic 30\%/70\% subdivision yields only
five tuning positives and 13 certification positives.  Five is below the
method's minimum of 15, and even zero incidents among 13 accepted certification
positives gives \(U(0,13;0.05)=0.206>0.15\).  The dataset therefore cannot
support the stated certificate under this design; detector competence is
reported as a second, independent diagnostic.

\section{Results}
\label{sec:results}

\subsection{Most primary policies do not pass, and simpler policies do better}

\begin{figure}[!htb]
  \centering
  \includegraphics[width=\linewidth]{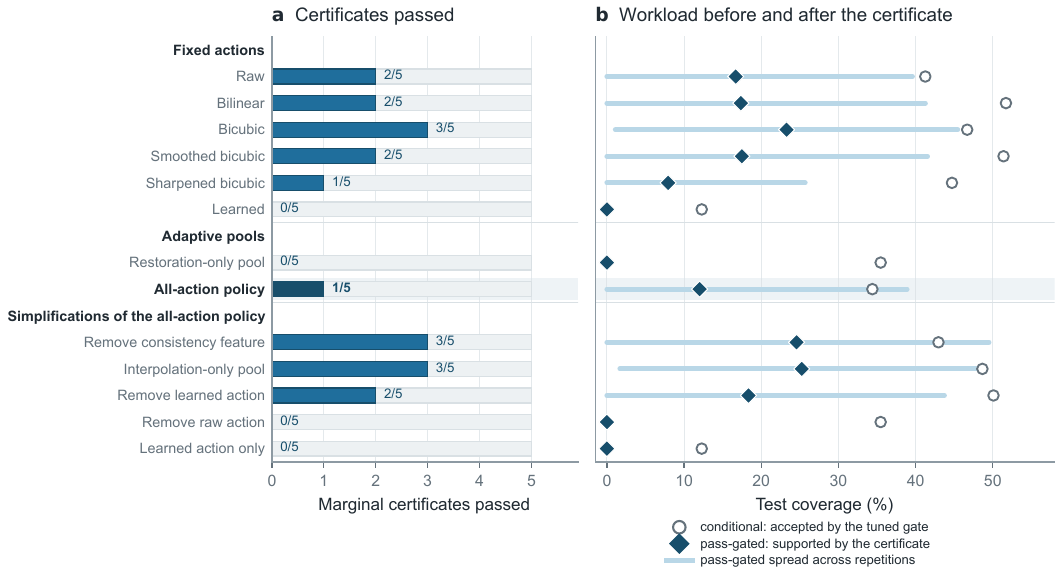}
  \caption{Policy comparison under the same five training repetitions.  Every
  row is fit, tuned, and certified separately; none inherits another row's
  certificate.  (a) Number of marginal certificates that pass.  (b) Test
  coverage before the certificate decision (conditional, open circles) and the
  coverage the decision actually supports (pass-gated, diamonds), with segments
  spanning one standard deviation across repetitions, clipped at zero.  The
  all-action policy (highlighted) passes fewer certificates than fixed bicubic
  and than three of its own simplifications.  The rows are separate marginal
  statements.}
  \label{fig:policy-comparison}
\end{figure}

\begin{table}[!htb]
\centering
\caption{Nominal Carinthia-S results at the 25\% recall floor for the primary
compact U-Net + residual-restorer family.  Each seed--policy pair has a separate
marginal certificate.  Conditional coverage is averaged before the certificate
decision; pass-gated coverage is set to zero when that policy fails.}
\label{tab:primary}
\small
\begin{tabular}{@{}lccc@{}}
\toprule
Policy & Pass/5 & Cond. coverage (\%) & Pass-gated coverage (\%) \\
\midrule
Raw & 2 & 41.3 & \(16.7\pm22.9\) \\
Bilinear & 2 & 51.7 & \(17.4\pm23.9\) \\
Bicubic & 3 & 46.7 & \(23.3\pm22.2\) \\
Smoothed bicubic & 2 & 51.4 & \(17.5\pm24.1\) \\
Sharpened bicubic & 1 & 44.7 & \(7.9\pm17.7\) \\
Learned & 0 & 12.3 & \(0.0\pm0.0\) \\
Restoration-only pool & 0 & 35.5 & \(0.0\pm0.0\) \\
\textbf{All-action policy} & \textbf{1} & \textbf{34.4} &
\(\mathbf{12.0\pm26.9}\) \\
\bottomrule
\end{tabular}
\end{table}

The primary all-action policy passes in one of five training repetitions
(\Cref{fig:policy-comparison} and \Cref{tab:primary}).  Seeds 201 and 203 have no tuning
prefix with the required
15 accepted positives.  Seed 207 fails both bounds, seed 223 fails
the excess-activation bound, and seed 211 passes both.  The seed-211 policy has
60.1\% test coverage and certification bounds
\(\ucbLoss=0.135\) and \(\ucbAct=0.113\).  Its held-out test incident rates are
11.3\% and 3.7\%, respectively.  These test rates describe the policy after
the decision; they do not replace the certification counts in
\Cref{tab:certificate-evidence}.

\begin{table}[!htb]
\centering
\caption{Certification evidence for the primary all-action policy.  A pass is
a marginal decision for the stated seed; the table is not a menu from which a
seed may be selected without multiplicity control.  Test incident counts are
descriptive outcomes after the certification decision.}
\label{tab:certificate-evidence}
\scriptsize
\setlength{\tabcolsep}{3.2pt}
\begin{tabular}{@{}rccccc@{}}
\toprule
Seed & Cert. loss \(e/n\;(U)\) & Cert. activation \(e/n\;(U)\) & Test loss \(e/n\) & Test activation \(e/n\) & Test C / result \\
\midrule
201 & -- (1.000) & -- (1.000) & -- & -- & 0.0\% / no gate \\
203 & -- (1.000) & -- (1.000) & -- & -- & 0.0\% / no gate \\
207 & 20/169 (0.167) & 21/182 (0.162) & 30/270 & 33/282 & 63.2\% / fails both \\
211 & 14/158 (0.135) & 12/169 (0.113) & 29/256 & 10/268 & 60.1\% / passes \\
223 & 1/123 (0.038) & 15/133 (0.168) & 5/205 & 7/217 & 48.7\% / fails activation \\
\bottomrule
\end{tabular}
\end{table}

\paragraph{Worked certificate reading.}
For seed 211, the tuning-selected gate accepts 158 positive images, of which 14
have evidence-loss incidents, and 169 images overall, of which 12 have
excess-activation incidents.  With endpoint error allocation
\(\delta/2=0.05\), the one-sided Clopper--Pearson calculations give
\(U(14,158;0.05)=0.135\) and \(U(12,169;0.05)=0.113\).  Both values are below
the 0.15 targets, so this fixed policy passes its marginal certificate.  Seed
223 illustrates why both endpoints are necessary: its evidence-loss bound is
0.038, but its excess-activation bound is 0.168 and the policy fails.

The key comparison is unfavorable to adaptivity
(\Cref{fig:policy-comparison}).  Fixed bicubic passes in three of five
repetitions and reaches \(23.3\%\pm22.2\%\) pass-gated coverage, compared with
1/5 and \(12.0\%\pm26.9\%\) for the all-action policy, whose pass count is also
exceeded by bilinear and smoothed bicubic.  Panel~(b) shows where the workload
is lost: conditional coverage is broadly similar across policies, between 34\%
and 52\% for everything except learned restoration, so the differences in
supported workload come from the certificate decision rather than from how much
each tuned gate accepts.  These separately certified comparisons do not prove
that bicubic is universally best, but they do not support adaptive-policy
superiority in the present study.

\subsection{Risk--coverage behavior and evidence sufficiency}

\begin{figure}[!htb]
  \centering
  \includegraphics[width=\linewidth]{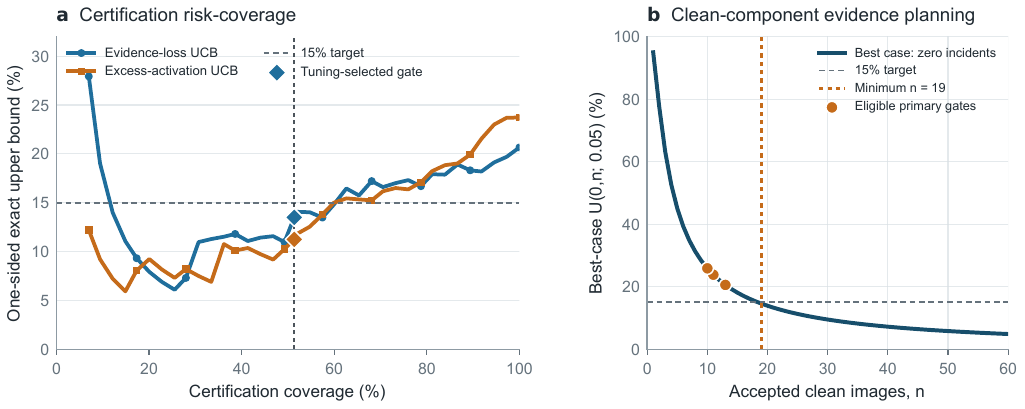}
  \caption{Certification behavior and an evidence-planning diagnostic.
  (a) One-sided exact evidence-loss and excess-activation upper bounds versus
  actual certification coverage for seed 211.  The vertical line and diamonds
  mark the gate selected on the disjoint tuning role; the curves at other
  prefixes are descriptive and were not candidate certificates.
  (b) Best-case clean-component bound \(U(0,n;0.05)\) as a function of the
  number of accepted clean images.  Dots mark accepted-clean counts for the
  three primary repetitions with eligible gates (seed 207: \(n=13\); seed 211:
  \(n=11\); seed 223: \(n=10\)); seeds 201 and 203 have no eligible gate.  This panel is a sample-planning diagnostic
  for the clean component, not a separately certified endpoint.  The primary
  excess-activation certificate is evaluated over all accepted images.}
  \label{fig:risk-coverage}
\end{figure}

\Cref{fig:risk-coverage}a separates the tuning decision from its held-out
evaluation.  At the tuning-selected gate, certification coverage is 51.4\% and
both upper bounds are below 15\%.  Expanding the accepted score prefix raises
coverage but eventually carries both bounds above target.  Because the off-gate
prefixes use certification outcomes, the curve diagnoses score ordering only.

Panel~(b) quantifies the information available for the clean-only activation
component used in prevalence standardization.  Even with zero incidents, at
least 19 accepted clean images are needed for a one-sided 95\% upper bound to
reach 0.15.  The eligible primary gates accept only 10--13 clean images on the
certification role.  Thus the clean component remains too imprecise for a
separate 15\% statement, even though the prespecified all-accepted activation
endpoint has a much larger denominator.
\subsection{Candidate fidelity and detector incidents are differently ordered}

\begin{figure}[!htb]
  \centering
  \includegraphics[width=\linewidth]{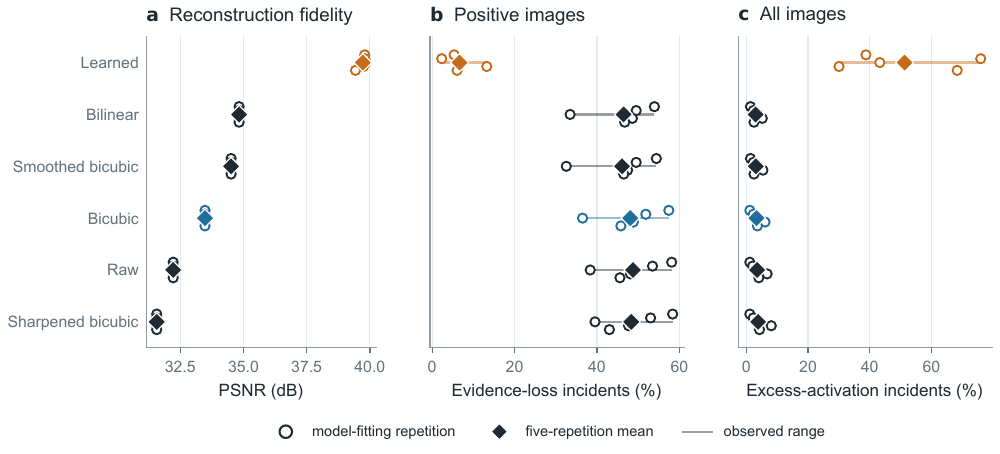}
  \caption{Fidelity and detector-relative incidents for the six candidate
  actions on the nominal test role.  Action rows are aligned across PSNR,
  evidence-loss incidence among positives, and excess-activation incidence among
  all images.  Open circles are individual training repetitions, diamonds are
  five-repetition means, and thin segments show the observed range.  High PSNR
  does not identify a dual-risk policy.}
  \label{fig:action-tradeoffs}
\end{figure}

The action pool nevertheless contains real tradeoffs
(\Cref{fig:action-tradeoffs}).  Learned restoration reaches mean PSNR
39.7~dB and reduces evidence-loss incidence to 6.7\%, compared with
46.1--48.7\% for raw and fixed interpolation.  Its mean excess-activation
incidence is 51.3\%, however, versus 3.0--3.8\% for those actions.  Sharpened
bicubic has lower fidelity but a different detector profile.  PSNR therefore
does not identify a policy satisfying both incident limits.  The practical
failure is not lack of candidate diversity; it is that the primary all-action policy
does not use that diversity reliably enough to pass more often than simple
fixed transformations.

\subsection{Fitted scores provide ordering, not probability calibration}

\begin{figure}[!htb]
  \centering
  \includegraphics[width=\linewidth]{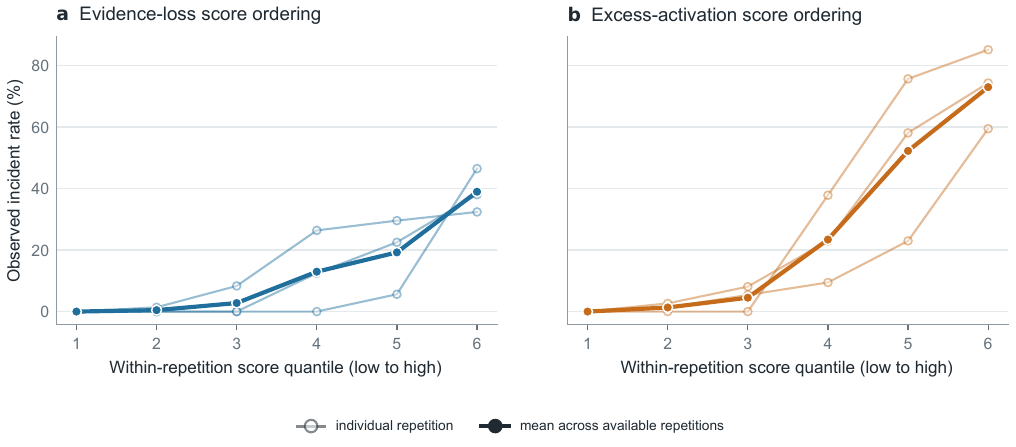}
  \caption{Observed incident rate across equal-frequency bins of each fitted
  endpoint score for the action selected on an image.  Bins are formed within
  each repetition having a defined gate.  Light open curves show individual
  repetitions and the thick filled curve their mean.  The horizontal
  coordinate is a within-repetition endpoint-score quantile, not a predicted
  probability or the minimax gate score; recurring identities make the
  cross-repetition curves descriptive.}
  \label{fig:risk-score-behavior}
\end{figure}

Incident rates generally increase from low to high score quantiles for both
endpoints (\Cref{fig:risk-score-behavior}), supporting use of the scores as
rankers.  The curves also vary materially by training repetition.  Because the
logistic fits use class balancing, their sigmoid outputs do not estimate the
study-population probability without further correction and calibration.  The
method accordingly uses only the induced ordering.  A gate can pass solely from
its held-out incident counts and exact bounds.

\subsection{Illustrative decisions connect the gate to image evidence}

\begin{figure}[!htbp]
  \centering
  \includegraphics[width=\linewidth]{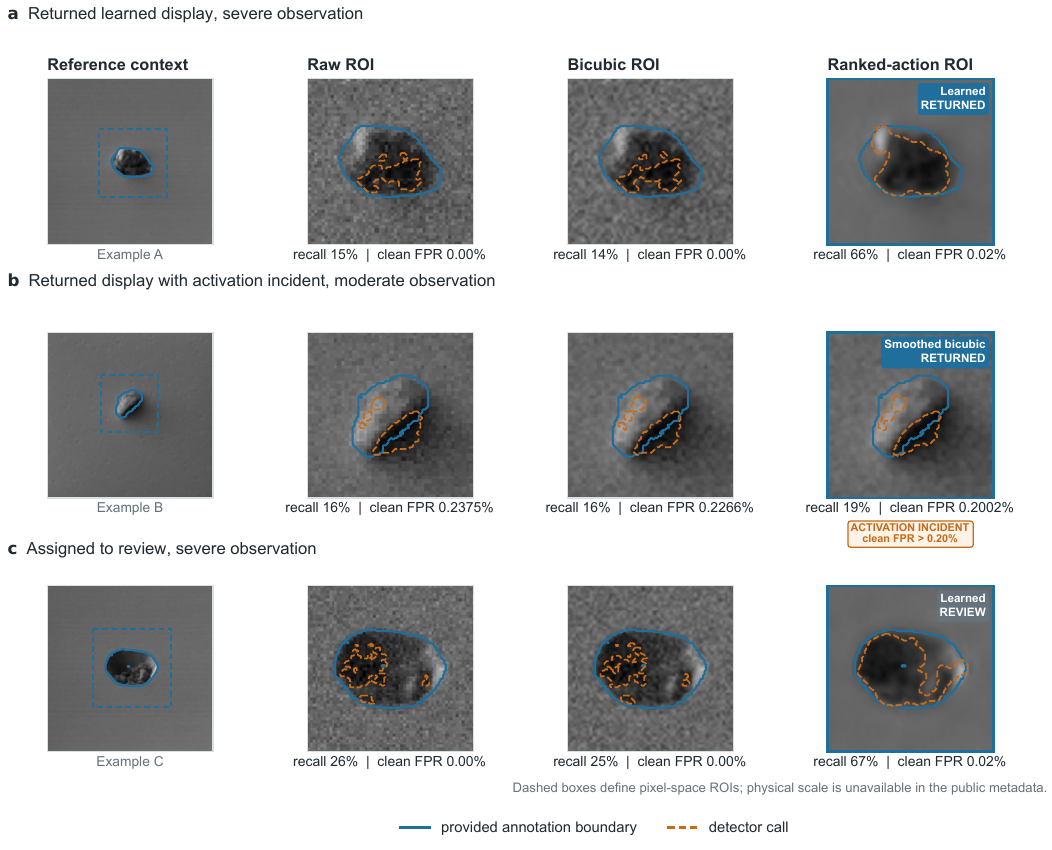}
  \caption{Illustrative test decisions for the seed-211 policy, the sole primary
  policy whose marginal certificate passes.  The reference context identifies
  each pixel-space region of interest; solid blue contours trace provided annotation
  boundaries and dashed orange contours are detector calls.  Metrics below each
  candidate are defect recall and clean-region FPR.  (a) Example~A: a learned
  action returned under severe degradation.  (b) Example~B: a returned
  smoothed-bicubic action whose clean-region FPR exceeds the 0.20\%
  excess-activation incident threshold.  (c) Example~C: a learned action ranked first
  but assigned to review.  The cases were selected to show these outcomes and
  are not a random sample.}
  \label{fig:qualitative-cases}
\end{figure}

\Cref{fig:qualitative-cases} shows why a policy needs both action choice and
review.  In panel (a), the learned candidate restores detector-supported defect
coverage relative to raw and bicubic.  Panel (b) shows that a passing marginal
policy can still make individual errors: the clean-region FPR crosses the
0.2\% incident threshold.  In panel (c), the ranked action improves recall but
its gate score remains outside the automatic-return region.  The certificate
controls a population rate, not every returned image.

\subsection{Feature and pool ablations expose unnecessary complexity}

\begin{table}[!htb]
\centering
\caption{Most informative policy simplifications under the same five seeds and
separate marginal certificates.  Values are pass-gated test coverage; the
complete ablation table and seed-level results appear in
\Cref{tab:supp-ablations,tab:supp-key-ablation-seeds}.}
\label{tab:key-ablations}
\small
\begin{tabular}{@{}lcc@{}}
\toprule
Policy variant & Pass/5 & Pass-gated coverage (\%) \\
\midrule
Full features, all actions & 1 & \(12.0\pm26.9\) \\
Remove consistency feature & 3 & \(24.6\pm24.9\) \\
Interpolation-only pool & 3 & \(25.2\pm23.6\) \\
Remove learned action & 2 & \(18.3\pm25.4\) \\
Remove raw action & 0 & \(0.0\pm0.0\) \\
Learned action only & 0 & \(0.0\pm0.0\) \\
\bottomrule
\end{tabular}
\end{table}

Removing the forward-consistency feature or restricting the pool to interpolation
raises the pass count from 1/5 to 3/5 (\Cref{tab:key-ablations}).  Conversely,
removing raw eliminates all passes, as does using only the learned action.  These
results show that the added candidates and features did not improve certification
in this design.  The pattern is consistent with, but does not establish, greater
score-model error from a larger action set.  Because all variants reuse identities and
were examined retrospectively, the table motivates a simpler prospective policy
rather than selecting the best variant from this study.

Supporting baselines reinforce the same point.  Entropy-gated raw display finds
no eligible threshold in any seed.  An outcome-informed oracle that chooses the
action with the smallest observed dual-incident indicator also fails an
accept-all certificate in every seed; for seed 211,
\(\ucbLoss=0.295\) (Supplementary \Cref{tab:supp-baseline-certificates}).  Even idealized action selection cannot compensate for a
gate that returns too many high-risk positives.

\subsection{Architecture and second-domain checks do not provide external confirmation}

The alternate attention-gated detector + U-Net restorer family passes 3/5
repetitions, with \(18.9\%\pm19.6\%\) pass-gated coverage.  Across its
passing repetitions, the unweighted mean held-out incident rates are 0.8\%
evidence loss and 4.4\% excess activation
(Supplementary \Cref{tab:supp-alternate-family}); denominators are not pooled.  The family
contrast shows architecture sensitivity, not architecture-independent validity;
each seed--family pair remains a separate marginal policy.

KolektorSDD supplies no external certificate: detector-validation Dice is only
approximately 0.033--0.042 across seeds, the split has five tuning positives
(below the required 15), and 13 certification positives cannot attain the 15\%
target even with zero incidents.  These detector-competence and evidence-volume
failures preclude a test of a successfully fitted gate on a new domain
(Supplementary \Cref{tab:supp-domain-diagnostics}).
\subsection{Severity and population shift reveal where return becomes unreliable}

\begin{figure}[!htb]
  \centering
  \includegraphics[width=\linewidth]{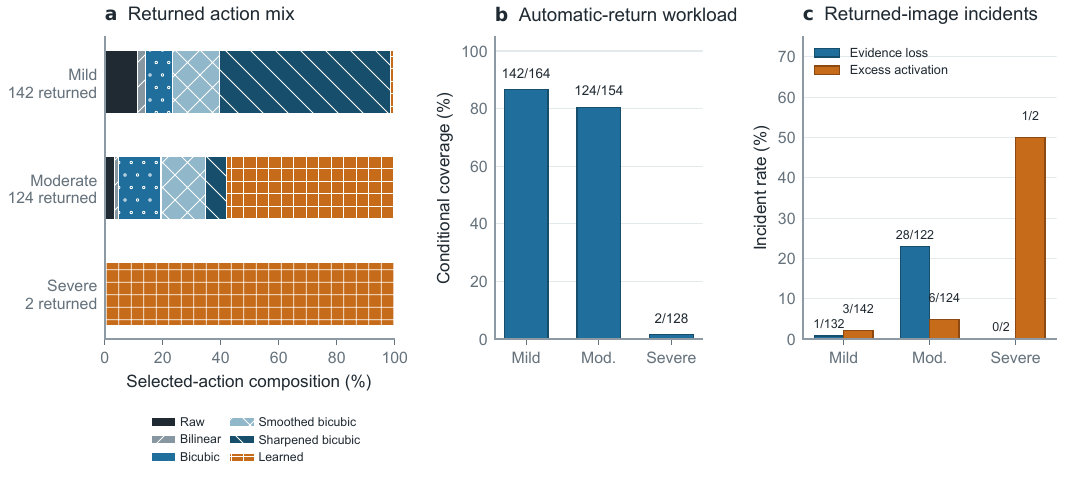}
  \caption{Behavior of the seed-211 policy by controlled degradation severity.
  (a) Action composition among returned images.  (b) Automatic-return workload
  with exact returned/total counts.  (c) Descriptive incident rates and
  endpoint-specific denominators among returned images.  Severity groups were
  not certified separately.}
  \label{fig:severity-behavior}
\end{figure}

The seed-211 gate returns 86.6\% of mild observations (142/164), 80.5\% of
moderate observations (124/154), and 1.6\% of severe observations (2/128)
(\Cref{fig:severity-behavior}).  Sharpened bicubic dominates mild returns
(84/142), while learned restoration dominates moderate returns (72/124).  Of
the two severe images returned, one incurs an excess-activation incident.  These
small subgroup counts emphasize why the aggregate certificate cannot be
reassigned to a severity stratum.  Panel (b) is the operational workload view:
most automatic returns occur on mild and moderate observations, while severe
cases are almost entirely reviewed.

\begin{figure}[!htb]
  \centering
  \includegraphics[width=\linewidth]{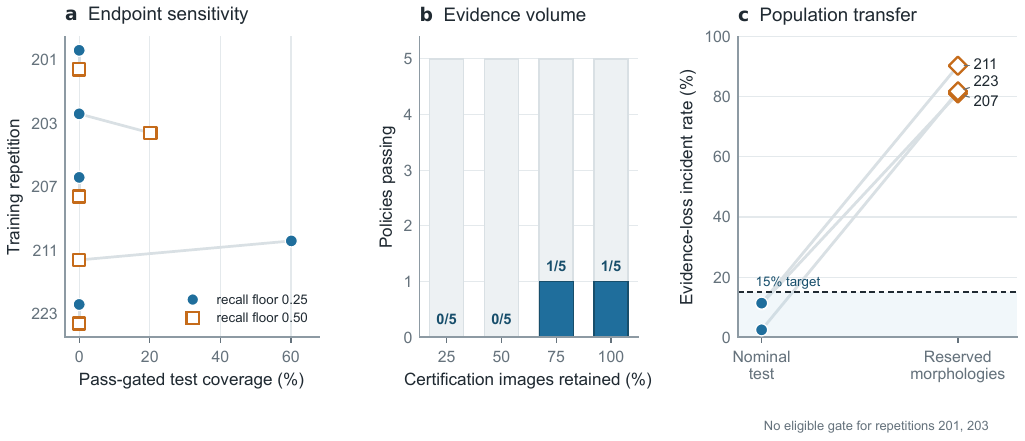}
  \caption{Certificate boundaries.  (a) Pass-gated coverage after refitting the
  score models and policy at recall floors 0.25 and 0.50.  (b) Number of policies
  passing when deterministic subsets of the fixed certification role are
  retained; this is a descriptive evidence-volume sensitivity.  (c)
  Evidence-loss incidence when nominal gates are applied to reserved morphologies.
  Repetitions without an eligible nominal gate are omitted from panel~(c) and
  noted in the figure.  No reserved-population point carries a transferred
  certificate.}
  \label{fig:boundary-evidence}
\end{figure}

At the 0.50 recall floor, one of five all-action policies still passes, but mean
pass-gated coverage falls to \(4.0\%\pm9.0\%\); seed 203, rather than seed 211,
passes (\Cref{fig:boundary-evidence}a).  Retaining 25\%, 50\%, 75\%, and 100\% of
the fixed certification role yields 0/5, 0/5, 1/5, and 1/5 passes
(\Cref{fig:boundary-evidence}b).  This monotone evidence limitation is expected:
small accepted-positive counts make an exact upper bound too wide even when the
observed incident rate is low.

Among seeds with a defined nominal gate, coverage on the 67 reserved
morphologies is 32.8--55.2\%, while evidence-loss incidence rises to
81.1--90.3\% (\Cref{fig:boundary-evidence}c).  This is direct empirical evidence
against transporting the nominal certificate.  Taken together, the main result
is a failure map rather than an adaptive-restoration victory: simple policies
are stronger on the nominal data, limited endpoint evidence constrains coverage,
and changed morphologies invalidate the nominal behavior.

\section{Discussion}
\label{sec:discussion}

\subsection{Certification changes the conclusion drawn from restoration results}

The candidate-level result alone would encourage an adaptive policy: learned
restoration reduces severe evidence-loss incidents, fixed interpolation limits
excess activation, and the preferred action changes with degradation
(\Cref{fig:action-tradeoffs,fig:severity-behavior}).  The policy-level result
reverses that impression.  Only one of five primary all-action policies passes
its marginal certificate, while fixed bicubic, an interpolation-only pool, and a
policy without the consistency feature pass in three of five repetitions
(\Cref{tab:primary,tab:key-ablations}).  Both levels of evidence are needed to
interpret the policy.
The action pool contains complementary transforms, but in these five
repetitions the resulting policy did not yield better certification outcomes
than the simpler alternatives.

This is a useful outcome for selective inspection.  A certificate is valuable
not only when it supports automatic return, but also when it exposes that model
complexity, training sensitivity, or limited positive evidence makes return
unsupported.  Review is therefore part of the policy rather than a failure to
produce a result.  The low pass-gated coverage summarizes how much workload the
available evidence supports after failed policies contribute zero; it is not an
estimate of the capacity of a future, better-designed policy.

\subsection{Interpretation of scores, incidents, and certificates}

The fitted sigmoid outputs are class-balanced ranking scores.  Their monotone
incident ordering is useful, but their numeric values are not calibrated
probabilities for the study population.  This distinction is important because
a visually plausible risk plot could otherwise be read as probability
calibration.  The held-out exact bounds, not the score scale, determine the
marginal certificate.

The incident names also bound the physical interpretation.  Evidence loss means
that the fixed detector recovers fewer annotated defect pixels than the stated
floor; it does not prove that a transformation erased a physical defect.
Excess activation means that detector calls occupy too much annotated clean
area; it does not prove that new physical structure was hallucinated.  Both
surrogates conflate image transformation behavior with detector competence.
KolektorSDD makes that dependence explicit: when detector Dice is near 0.04, a
restoration-policy certificate cannot substitute for a competent downstream
measurement system.

\Cref{prop:false-certification} controls a repeated-sampling false-certification
event for one fixed policy under the working model.  It does not guarantee every
returned image, provide a posterior probability of safety, or certify a subgroup,
and it does not support selecting seed 211 because that seed passed.  The
pass-count analysis is diagnostic, subject to the policy-family boundary of
\Cref{sec:guarantee}.

The observed pass/fail pattern is not specific to the reported
Clopper--Pearson/Bonferroni construction.  The certificates were recomputed from
the same counts under a \v{S}id\'ak allocation, under no allocation at all, and
under a Wilson score limit (Supplementary \Cref{tab:supp-bound-sensitivity}).  Across these examined
alternatives, the pass/fail pattern is unchanged.  Removing the multiplicity
allocation entirely moves the largest bound by about one percentage point,
whereas the failing bounds exceed their target by 0.012--0.018, and two
repetitions fail earlier still, at threshold tuning, where no bound is involved.
These checks show that the reported failures are not rescued by removing the
allocation or by the Wilson comparison; they do not establish invariance to
every valid interval procedure.  In this comparison, the margins are governed
mainly by the observed incident counts and endpoint evidence.

\subsection{Population mixture and evidence requirements}

Separating accepted positives from all accepted images prevents clean prevalence
from diluting evidence-loss risk.  Excess activation remains mixture-dependent,
which is why \Cref{eq:standardized-activation} reports component risks and an
operating-prevalence standardization.  This issue is consequential here: the
nominal test set has 430 positives but only 16 clean images, and calibration has
23 clean images.  Clean coverage and clean-component activation estimates are
therefore imprecise and morphology-linked.  Exact bounds make the requirement
visible: with a 15\% target and endpoint error allocation 0.05, even zero
incidents require at least 19 accepted units, and one incident requires 30
(Supplementary \Cref{tab:supp-sample-size}).  KolektorSDD's 13 certification positives cannot meet
the target for that reason alone.  Positive and clean counts, not only total
image count, must therefore drive prospective sample-size planning.

\subsection{Operational implications}

A prospective study should prespecify the physical sampling unit, endpoints,
strata, and review costs.  Because the public metadata do not expose wafer or lot
groupings, the present image-level i.i.d.\ model cannot address within-cluster
dependence; new data should reserve entire known clusters for certification.
Detector-relative recall is not physical preservation, so that stronger endpoint
would require paired acquisition or independent metrology.  Review should be
measured rather than treated as costless abstention, consistent with
learning-to-defer formulations \cite{mozannar2020defer,narasimhan2022defer} and
the broader SEM-inspection literature \cite{sem_review2023}.  These are design
requirements for a follow-up, not claims about current deployment.

The present policy generates every candidate and detector map before selection,
so it makes no latency or energy claim.  A cascade would be a different policy
requiring new threshold selection and certification.

\subsection{Limitations and prospective study design}

The evidence is retrospective: the Carinthia-S identity bank was previously
studied in \cite{yang2026waferinspectsr} (\Cref{tab:prior-comparison}).  New data
roles and policy records separate the estimand, but they cannot recreate
pristine external confirmation.  Physical independence is unknown because public
metadata do not expose wafer, lot, specimen, or field-of-view grouping.
Treating images as i.i.d.\ may be anti-conservative if correlated images cross
data roles.

The blur--downsample--noise schedule provides exact correspondence but is not a
named SEM degradation process, and physical scale is unavailable.  The study
therefore cannot establish behavior under real rescanning, charging, drift,
contamination, or recipe change.  Reserved Carinthia-S morphologies show
catastrophic transfer, while KolektorSDD is underpowered and has an inadequate
detector.  The 15\% risk limits and 25\% recall floor are research choices, not
manufacturing requirements, and the evaluation was not preregistered.

Finally, two compact architecture families and a six-feature logistic ranker do
not establish method generality.  Current restoration systems, alternative
detectors, calibrated or nonparametric rankers, and more efficient risk-control
procedures may improve the coverage--risk tradeoff.  The immediate prospective
baseline should nevertheless be the simplest supported policy---fixed bicubic or
an interpolation-only pool---rather than the primary all-action policy, followed by a fresh,
cluster-aware certificate on newly acquired data.

\section{Conclusion}
\label{sec:conclusion}

\method treats restoration as a selective image action whose population-level
incident rates must be supported before automatic return.  Action-specific
scores rank the raw display and five restored candidates; a disjoint sample then
provides exact marginal bounds for evidence loss among accepted positives and
excess activation among all accepted images.  The resulting statement is
transparent but narrow: it applies to one policy fixed in advance, one nominal
population, and an image-level i.i.d.\ working model.

On the retrospective Carinthia-S study, the protocol yields auditable
risk--coverage behavior.  The primary all-action policy passes in one of five
repetitions and reaches
\(12.0\%\pm26.9\%\) pass-gated coverage.  Fixed bicubic, an interpolation-only
pool, and a reduced-feature policy each pass in three of five repetitions and
achieve higher mean pass-gated coverage.  Reserved morphologies produce
81.1--90.3\% evidence-loss incidence, and KolektorSDD cannot support the target
because both detector competence and positive sample size are inadequate.  The
evidence therefore supports certification and review as safeguards, but not
adaptive-policy superiority or cross-domain readiness
(\Cref{tab:primary,tab:key-ablations} and \Cref{fig:boundary-evidence}).

The next study should fix a simple policy before certification, reserve newly
acquired wafer- or lot-level clusters, elicit incident limits from manufacturing
costs, collect adequate positive and clean populations, and measure human-review
performance.  Until such evidence exists, restoration should remain an explicit,
selectable, and rejectable transformation rather than hidden mandatory
preprocessing.

\section*{Data and Code Availability}
Carinthia-S and KolektorSDD are public datasets distributed by their original
authors under their stated licenses.  Analysis code is available at
\url{https://github.com/nbbllxx0/SafeRestore}.  Deterministic split manifests,
image--action records, certification summaries, tests, and the source data
and scripts used for every reported figure will be released in that repository.
No model, operating point, numerical result, or synthetic record from
\cite{yang2026waferinspectsr} enters fitting, threshold selection, or
certification here; Carinthia-S identity overlap with that study is disclosed
in \Cref{tab:prior-comparison}.
Controlled degradations are counterfactual restoration tests; no proprietary
manufacturing imagery is used.

\section*{Declaration of Generative AI Assistance}
The Cursor editor assisted with code drafting and language editing.  The
authors designed the study, selected and verified the data, ran and checked the
experiments, reviewed the manuscript, and take responsibility for its content.

\bibliographystyle{unsrt}
\begingroup
\fontsize{7.6pt}{8.2pt}\selectfont
\bibliography{references}
\endgroup

\clearpage
\setcounter{section}{0}
\setcounter{table}{0}
\setcounter{figure}{0}
\renewcommand{\thesection}{S\arabic{section}}
\renewcommand{\thetable}{S\arabic{table}}
\renewcommand{\thefigure}{S\arabic{figure}}
\renewcommand{\theHsection}{supp.\arabic{section}}
\renewcommand{\theHtable}{supp.\arabic{table}}
\renewcommand{\theHfigure}{supp.\arabic{figure}}
\pdfbookmark[0]{Supplementary Material}{supplementary-material}
\section*{Supplementary Material}

\section{Analysis Scope and Statistical Interpretation}

The analysis treats raw display and five restoration candidates as one action
set.  Evidence loss is defined only on defect-positive images and counted among
accepted positives.  Excess activation is defined from clean-region detector
calls and counted among all accepted images.  Threshold-tuning and certification
roles are separated by a deterministic, seed-independent image-identifier rule.
A failed marginal certificate leaves the conditional gate coverage available for
diagnosis but gives zero pass-gated coverage.

The five training repetitions reuse image identities, observation assignments,
and tune/certify membership.  They quantify sensitivity to model fitting; they
are not five independent samples from a factory.  Each seed--policy pair has a
separate marginal certificate.  Pass counts do not justify choosing a passing
seed without multiplicity control or a fresh certification sample.

The Carinthia-S study is retrospective.  All 4,591 identities also appeared in
\cite{yang2026waferinspectsr} (\Cref{tab:prior-comparison}),
although no checkpoint, operating point, split, or numeric result
from that work is used here.  The analysis was not externally preregistered.
These constraints do not change the mechanical split-sample calculation, but
they limit the strength of empirical confirmation.

\section{Data Inventory and Split Verification}

\begin{table}[!htbp]
\centering
\caption{Carinthia-S inventory.  One image is the statistical unit; physical
wafer, lot, and specimen groups are unavailable.}
\label{tab:supp-carinthia}
\begin{tabular}{@{}lrrr@{}}
\toprule
Partition & Images & Positive & Clean \\
\midrule
Training & 3,186 & 3,014 & 172 \\
Detector validation & 431 & 416 & 15 \\
Calibration & 461 & 438 & 23 \\
Nominal test & 446 & 430 & 16 \\
Reserved morphology & 67 & 67 & 0 \\
\midrule
Total & 4,591 & 4,365 & 226 \\
\bottomrule
\end{tabular}
\end{table}

\Cref{tab:supp-carinthia} verifies the image inventory.  Classes 3, 4, and 6
define the nominal population; classes 1, 2, and 5 remain
intact as the reserved morphology diagnostic.  Positivity is determined by a
nonempty binary mask after nearest-neighbor mask resizing.  Dataset files are
read-only during analysis, and the archive records inventory and split digests
\cite{carinthias2025}.

\begin{table}[!htbp]
\centering
\caption{KolektorSDD item-level diagnostic split.  It is not the official fold
and cannot support the stated certificate because the tuning and certification
positive counts are too small.}
\label{tab:supp-kolektor}
\begin{tabular}{@{}lrrr@{}}
\toprule
Partition & Items & Images & Positive \\
\midrule
Training & 14 & 111 & 15 \\
Detector validation & 6 & 48 & 6 \\
Calibration & 17 & 136 & 18 \\
Test & 13 & 104 & 13 \\
\midrule
Total & 50 & 399 & 52 \\
\bottomrule
\end{tabular}
\end{table}

\Cref{tab:supp-kolektor} records the item-level split used only for the
second-dataset boundary diagnosis.

\section{Implementation Details}

Images are converted to grayscale and resized to \(256\times256\) with bilinear
interpolation; masks use nearest-neighbor interpolation followed by binary
thresholding.  Training augmentation consists only of independent random
horizontal and vertical flips.  No intensity jitter, crop augmentation, or
learning-rate scheduler is used.  Data loading uses batch size 8 and deterministic
random seeds; PyTorch deterministic cuDNN mode is enabled and benchmarking is
disabled.

The primary compact three-level U-Net detector has three encoder resolutions
with base width 16, group normalization, SiLU activations, 0.05 dropout,
transposed-convolution upsampling, and 117,393 trainable parameters.  It is
trained for five epochs with AdamW (learning rate \(2\times10^{-3}\), weight
decay \(10^{-4}\)).  The loss is weighted binary cross-entropy plus soft Dice;
the batch-specific positive weight is the clean-to-positive pixel ratio clipped
to \([1,80]\).  The checkpoint with highest mean detector-validation Dice is
retained.  The detector threshold is the higher empirical \((1-10^{-3})\)
quantile of detector scores on annotated clean pixels in the detector-validation
role.

The primary six-block residual restorer has a 32-channel input convolution, six
residual blocks, and an output residual added to bicubic interpolation (111,585
trainable parameters).  It is trained for five epochs with AdamW (learning rate
\(10^{-3}\), weight decay \(10^{-5}\)).  The objective is image \(\ell_1\) loss
plus 0.25 times the sum of horizontal and vertical gradient \(\ell_1\) losses.
The final epoch is used; no validation-based restorer checkpoint selection is
performed.  The alternate family uses an attention-gated U-Net detector and a
one-level U-Net residual restorer with otherwise matched training roles.

For each action and endpoint, a StandardScaler is fitted on the applicable
detector-validation population.  LogisticRegression uses the liblinear solver,
class weights ``balanced,'' \(C=1\), at most 2,000 iterations, and the training
seed as its random state.  Evidence-loss models use positive images only;
activation models use all images.  If an endpoint population contains only one
class, its fitted score is the observed constant label.  Candidate ties use the
order raw, bilinear, bicubic, smoothed bicubic, sharpened bicubic, learned.

Threshold tuning requires at least 15 accepted images and 15 accepted positives.
The largest ordered-score prefix whose two empirical incident rates do not
exceed half their 0.15 targets fixes the gate.  Certification uses one-sided
Clopper--Pearson bounds with endpoint error allocation 0.05
\cite{clopper1934}.  Zero accepted units give upper bound one.

The archived experiment summaries record Python 3.12.13, PyTorch
2.11.0+cu128, torchvision 0.26.0+cu128, NumPy 2.4.3, SciPy 1.17.1,
scikit-learn 1.8.0, Matplotlib 3.10.9, Pillow 12.1.1, and an NVIDIA RTX~5090.
Risk-score fitting, exact bounds, tables, and figures are CPU-capable; detector
and restorer training used the GPU.  All candidates are generated before action
selection, so runtime is not an efficiency result.

\section{Architecture-Family Sensitivity}

\begin{table}[!htbp]
\centering
\caption{Seed-level marginal certificates for the alternate family at the 25\% recall floor.  Outcomes do not authorize post-certification seed selection.}
\label{tab:supp-alternate-family}
\scriptsize
\setlength{\tabcolsep}{2.4pt}
\begin{tabular}{@{}rccccccc@{}}
\toprule
Seed & Cert. loss & \(\ucbLoss\) & Cert. act. & \(\ucbAct\) & Test loss & Test act. & Test C / result \\
\midrule
201 & 0/34 & 0.084 & 2/42 & 0.142 & 0/63 & 1/73 & 16.4\% / passes \\
203 & 18/148 & 0.175 & 18/161 & 0.161 & 23/228 & 24/240 & 53.8\% / fails both \\
207 & 4/105 & 0.085 & 10/113 & 0.145 & 0/169 & 14/179 & 40.1\% / passes \\
211 & 20/165 & 0.171 & 16/173 & 0.137 & 24/265 & 12/275 & 61.7\% / fails loss \\
223 & 1/96 & 0.048 & 8/109 & 0.129 & 4/157 & 7/169 & 37.9\% / passes \\
\bottomrule
\end{tabular}
\end{table}

\Cref{tab:supp-alternate-family} shows 3/5 passes and
\(18.9\%\pm19.6\%\) pass-gated coverage.  Across passing seeds 201, 207, and
223, the unweighted means of the held-out incident proportions are 0.8\%
evidence loss and 4.4\% excess activation; denominators are not pooled.  The differing pattern indicates
architecture sensitivity, not family-independent validity or a basis for
post-certification selection.
\section{Exploratory Tuning-Rule Sensitivity}

The half-margin criterion described in the main article is the only tuning rule
specified in the frozen analysis protocol.  After the primary results were
available, we recomputed thresholds with two alternatives: empirical incident
rates at the target itself and exact upper bounds on the tuning role.  The same
certification sample then evaluates every resulting gate
(\Cref{tab:supp-tuning-sensitivity}).

\begin{table}[!htbp]
\centering
\caption{Post-hoc tuning-rule sensitivity for the all-action policy.  Test
coverage is conditional on the reported gate and is shown even when the
certificate fails.  This shared-sample comparison is descriptive; it cannot be
used to select a rule and retain the primary confirmatory interpretation.}
\label{tab:supp-tuning-sensitivity}
\footnotesize
\renewcommand{\arraystretch}{1.05}
\begin{tabular}{@{}lrrrrlr@{}}
\toprule
Rule & Seed & \(\widehat\theta\) & \(\ucbLoss\) & \(\ucbAct\) & Decision & Test C (\%) \\
\midrule
Half-margin & 201 & -- & 1.000 & 1.000 & no gate & 0.0 \\
 & 203 & -- & 1.000 & 1.000 & no gate & 0.0 \\
 & 207 & 3.287 & 0.167 & 0.162 & fail & 63.2 \\
 & 211 & 3.530 & 0.135 & 0.113 & pass & 60.1 \\
 & 223 & 2.854 & 0.038 & 0.168 & fail & 48.7 \\
\addlinespace[2pt]
Empirical target & 201 & 2.300 & 0.054 & 0.259 & fail & 28.0 \\
 & 203 & 2.314 & 0.080 & 0.252 & fail & 17.5 \\
 & 207 & 4.638 & 0.129 & 0.284 & fail & 78.3 \\
 & 211 & 6.667 & 0.207 & 0.235 & fail & 100.0 \\
 & 223 & 3.035 & 0.053 & 0.213 & fail & 55.2 \\
\addlinespace[2pt]
Tuning UCB & 201 & -- & 1.000 & 1.000 & no gate & 0.0 \\
 & 203 & -- & 1.000 & 1.000 & no gate & 0.0 \\
 & 207 & 3.287 & 0.167 & 0.162 & fail & 63.2 \\
 & 211 & 3.530 & 0.135 & 0.113 & pass & 60.1 \\
 & 223 & 2.703 & 0.027 & 0.106 & pass & 45.3 \\
\bottomrule
\end{tabular}
\end{table}

The tuning-UCB variant yields two passing marginal certificates in this
retrospective comparison, but that observation occurred after the rule family
and certification outcomes were examined.  It therefore does not nominate the
tuning-UCB rule, seed 223, or any other row for deployment.  A future comparison
must prespecify one rule, apply familywise error control, or evaluate a selected
rule on fresh certification data.

\section{Complete Policy Comparisons}

\Cref{tab:supp-primary-all,tab:supp-strict-all} report the complete fixed-action
and adaptive-pool comparisons at both recall floors.

\begin{table}[!htbp]
\centering
\caption{Primary 25\% recall-floor results.  Pass-gated coverage assigns zero to
a failed marginal certificate.}
\label{tab:supp-primary-all}
\small
\begin{tabular}{@{}lrrrr@{}}
\toprule
Policy & Pass/5 & Conditional C (\%) & Pass-gated C (\%) & SD \\
\midrule
Raw & 2 & 41.3 & 16.7 & 22.9 \\
Bilinear & 2 & 51.7 & 17.4 & 23.9 \\
Bicubic & 3 & 46.7 & 23.3 & 22.2 \\
Smoothed bicubic & 2 & 51.4 & 17.5 & 24.1 \\
Sharpened bicubic & 1 & 44.7 & 7.9 & 17.7 \\
Learned & 0 & 12.3 & 0.0 & 0.0 \\
Restoration-only pool & 0 & 35.5 & 0.0 & 0.0 \\
All-action policy & 1 & 34.4 & 12.0 & 26.9 \\
\bottomrule
\end{tabular}
\end{table}

\begin{table}[!htbp]
\centering
\caption{Policy results after refitting at the stricter 50\% recall floor.}
\label{tab:supp-strict-all}
\small
\begin{tabular}{@{}lrrr@{}}
\toprule
Policy & Pass/5 & Conditional C (\%) & Pass-gated C (\%) \\
\midrule
Raw & 0 & 28.4 & 0.0 \\
Bilinear & 2 & 28.4 & 13.3 \\
Bicubic & 2 & 27.5 & 12.1 \\
Smoothed bicubic & 3 & 28.2 & 20.3 \\
Sharpened bicubic & 1 & 27.7 & 6.7 \\
Learned & 0 & 10.2 & 0.0 \\
Restoration-only pool & 1 & 17.9 & 4.0 \\
All-action policy & 1 & 18.0 & 4.0 \\
\bottomrule
\end{tabular}
\end{table}

\begin{table}[!htbp]
\centering
\caption{Action counts below the primary all-action gate on the nominal test
role.  Seeds 201 and 203 have no eligible gate; only seed 211 passes its marginal
certificate.}
\label{tab:supp-action-mix}
\small
\begin{tabular}{@{}rrrrrrrrr@{}}
\toprule
Seed & Pass & Returned & Raw & Bil. & Bic. & Smooth & Sharp & Learned \\
\midrule
201 & no & 0 & -- & -- & -- & -- & -- & -- \\
203 & no & 0 & -- & -- & -- & -- & -- & -- \\
207 & no & 282 & 25 & 28 & 29 & 59 & 41 & 100 \\
211 & yes & 268 & 20 & 6 & 31 & 42 & 93 & 76 \\
223 & no & 217 & 11 & 12 & 18 & 9 & 147 & 20 \\
\bottomrule
\end{tabular}
\end{table}

\begin{table}[!htbp]
\centering
\caption{Seed-211 component and prevalence-standardized coverage for the
primary all-action policy.  Standardized coverage combines positive and clean
component coverage at the stated operating prevalence \(\pi\).}
\label{tab:supp-component-coverage}
\small
\begin{tabular}{@{}lcc@{}}
\toprule
Quantity & Certification & Test \\
\midrule
Positive coverage & 158/310 = 0.510 & 256/430 = 0.595 \\
Clean coverage & 11/19 = 0.579 & 12/16 = 0.750 \\
Standardized, \(\pi=0.01\) & 0.578 & 0.748 \\
Standardized, \(\pi=0.10\) & 0.572 & 0.735 \\
Standardized, \(\pi=0.50\) & 0.544 & 0.673 \\
\bottomrule
\end{tabular}
\end{table}

\Cref{tab:supp-action-mix,tab:supp-component-coverage} show that seed 211 accepts
158/310 certification positives and 11/19 clean images.  Only 16 test clean
images are available, so the clean component should not be treated as precisely
estimated.

\begin{table}[!htbp]
\centering
\caption{Primary-family supporting baselines.}
\label{tab:supp-baselines}
\small
\begin{tabular}{@{}lrr@{}}
\toprule
Baseline & Pass/5 & Pass-gated C (\%) \\
\midrule
All-action policy & 1 & \(12.0\pm26.9\) \\
Entropy-gated raw & 0 & \(0.0\pm0.0\) \\
Outcome-informed action oracle with accept-all gate & 0 & \(0.0\pm0.0\) \\
\bottomrule
\end{tabular}
\end{table}

\begin{table}[!htbp]
\centering
\caption{Seed-level certification evidence for supporting baselines.  Entropy
gating finds no eligible tuning threshold.  The outcome-informed oracle accepts
all certification images after choosing actions with observed outcomes.}
\label{tab:supp-baseline-certificates}
\footnotesize
\begin{tabular}{@{}rcccc@{}}
\toprule
Seed & Entropy gate & Oracle loss \(e/n\;(U)\) & Oracle activation \(e/n\;(U)\) & Oracle result \\
\midrule
201 & none & 95/310 (0.352) & 4/329 (0.028) & fails loss \\
203 & none & 151/310 (0.535) & 4/329 (0.028) & fails loss \\
207 & none & 88/310 (0.329) & 7/329 (0.040) & fails loss \\
211 & none & 78/310 (0.295) & 19/329 (0.084) & fails loss \\
223 & none & 130/310 (0.467) & 9/329 (0.047) & fails loss \\
\bottomrule
\end{tabular}
\end{table}

The oracle in \Cref{tab:supp-baselines,tab:supp-baseline-certificates} chooses the action with the smallest
observed dual-incident indicator
for each image and then evaluates an accept-all gate.  It is not an
application-time comparator because it uses outcomes.  Even this construction has
\(\ucbLoss>0.15\) in every seed (0.295 for seed 211), showing that action choice
alone cannot repair an uncertified return region.

\section{Feature and Action-Pool Ablations}

\begin{table}[!htbp]
\centering
\caption{All feature and action-pool ablations.  Each variant is tuned and
certified as a separate marginal policy on the same five identities and
seeds.}
\label{tab:supp-ablations}
\small
\begin{tabular}{@{}llrrr@{}}
\toprule
Type & Variant & Pass/5 & Mean pass-gated C (\%) & SD \\
\midrule
Feature & full & 1 & 12.0 & 26.9 \\
Feature & no consistency & 3 & 24.6 & 24.9 \\
Feature & no disagreement & 2 & 20.5 & 28.2 \\
Feature & no entropy & 2 & 16.5 & 23.0 \\
Feature & no evidence shift & 2 & 20.5 & 28.4 \\
Feature & no probability area & 0 & 0.0 & 0.0 \\
Feature & no residual & 2 & 10.9 & 16.4 \\
Feature & uncertainty only & 2 & 16.5 & 28.1 \\
Pool & all actions & 1 & 12.0 & 26.9 \\
Pool & interpolation only & 3 & 25.2 & 23.6 \\
Pool & learned only & 0 & 0.0 & 0.0 \\
Pool & no learned & 2 & 18.3 & 25.4 \\
Pool & no raw & 0 & 0.0 & 0.0 \\
Pool & no sharpening & 2 & 21.0 & 29.3 \\
Pool & no smoothing & 1 & 12.4 & 27.8 \\
\bottomrule
\end{tabular}
\end{table}

\begin{table}[!htbp]
\centering
\caption{Seed-level pass pattern for the strongest simplifications.  The
patterns show why aggregate pass counts must not be interpreted as five
independent replications.}
\label{tab:supp-key-ablation-seeds}
\small
\begin{tabular}{@{}lrrrrrr@{}}
\toprule
Variant & 201 & 203 & 207 & 211 & 223 & Pass/5 \\
\midrule
Full / all actions & 0 & 0 & 0 & 1 & 0 & 1 \\
No consistency & 0 & 1 & 0 & 1 & 1 & 3 \\
Interpolation only & 1 & 0 & 0 & 1 & 1 & 3 \\
No learned action & 1 & 0 & 0 & 0 & 1 & 2 \\
No raw action & 0 & 0 & 0 & 0 & 0 & 0 \\
\bottomrule
\end{tabular}
\end{table}

The variants in \Cref{tab:supp-ablations,tab:supp-key-ablation-seeds} were
examined on the same image bank and therefore motivate a
future policy; they do not license selecting the best variant from this table.
In particular, the consistency feature and learned candidate add complexity
without improving the observed certificate pattern, while availability of the
raw action is essential in this design.

\section{Evidence Volume and Population Boundaries}

\begin{table}[!htbp]
\centering
\caption{Marginal all-action certificates when deterministic subsets of the
fixed certification role are retained.  This is a descriptive evidence-volume
sensitivity, not a new random-sampling experiment.}
\label{tab:supp-cert-size}
\begin{tabular}{@{}rrrr@{}}
\toprule
Retained (\%) & Pass/5 & \(\nLoss\) range & \(\nAct\) range \\
\midrule
25 & 0 & 0--42 & 0--46 \\
50 & 0 & 0--79 & 0--85 \\
75 & 1 & 0--123 & 0--134 \\
100 & 1 & 0--169 & 0--182 \\
\bottomrule
\end{tabular}
\end{table}

\Cref{tab:supp-cert-size} isolates the effect of reducing the certification
role while leaving each previously tuned gate fixed.

\begin{table}[!htbp]
\centering
\caption{Nominal all-action gates applied unchanged to reserved morphologies.
No row has a transferred certificate.}
\label{tab:supp-shift}
\begin{tabular}{@{}rrrr@{}}
\toprule
Seed & Coverage (\%) & Evidence loss (\%) & Excess activation (\%) \\
\midrule
207 & 55.2 & 81.1 & 0.0 \\
211 & 46.3 & 90.3 & 0.0 \\
223 & 32.8 & 81.8 & 4.5 \\
\bottomrule
\end{tabular}
\end{table}

\Cref{tab:supp-shift} shows a direct failure of transport: the nominal
gates still return many images, but evidence-loss incidence is above 80\% for all
three defined gates.  Coverage alone would therefore give a misleading picture
of transfer.

\section{Clean-Proxy and KolektorSDD Diagnostics}

The study contains few nominal clean images and they are linked to class 6.  A
post-hoc sensitivity was therefore run on 80 inpainted clean twins of positive
test images, chosen as the first 80 positives in deterministic manifest order.
The originally designated proxy seed (201) had no eligible gate.  The analysis
was then repeated descriptively for seed 211 because it was the sole passing
primary policy; this choice occurred after the certificate outcomes and carries
no confirmatory claim.  Under that gate, 38/80 proxies are returned (47.5\%), and
10/38 returned proxies have excess activation (26.3\%).  The one-sided 95\%
upper bounds are 0.573 for return frequency and 0.405 for activation frequency.
Inpainting does not create matched physical clean acquisitions, and the selected
80 identities are not a random or morphology-balanced sample.

KolektorSDD fails before an informative restoration-policy comparison is
possible.  The powered, non-official item-level diagnostic split has five
tuning positives and 13 certification positives.  These results identify
sample-size and detector-competence boundaries; they are not evidence for
external-domain performance.

\begin{table}[!htbp]
\centering
\caption{Clean-proxy and powered KolektorSDD diagnostic evidence.  The proxy is
post hoc; the Kolektor split is a larger non-official item-level diagnostic split,
not official fold~0.}
\label{tab:supp-domain-diagnostics}
\small
\begin{tabular}{@{}lcc@{}}
\toprule
Diagnostic & Count or range & Rate / one-sided 95\% upper bound \\
\midrule
Clean-proxy return, seed 211 & 38/80 & 47.5\% / 0.573 \\
Clean-proxy excess activation & 10/38 & 26.3\% / 0.405 \\
Kolektor tuning positives & 5 & below minimum 15 \\
Kolektor certification positives & 13 & \(U(0,13;0.05)=0.206\) \\
Kolektor detector-validation Dice & five seeds & 0.033--0.042 \\
\bottomrule
\end{tabular}
\end{table}

\Cref{tab:supp-domain-diagnostics} collects the governing counts and bounds.

\section{Sensitivity to the Bound Construction}

The reported certificates combine an exact one-sided Clopper--Pearson bound with
a Bonferroni allocation of \(\delta=0.10\) across the two endpoints.  Both
choices are conservative, so it is fair to ask how much of the observed failure
pattern they cause.  \Cref{tab:supp-bound-sensitivity} recomputes the all-action
certificates from the same certification counts under three alternatives: a
\v{S}id\'ak allocation \(\gamma=1-(1-\delta)^{1/2}=0.0513\); no multiplicity
allocation at all, that is, \(\gamma=\delta=0.10\) for each endpoint separately,
which is an upper limit on what any sharper allocation could recover; and a
Wilson score upper limit at \(\gamma=0.05\), which drops exactness in exchange
for a shorter interval.

\begin{table}[!htbp]
\centering
\caption{All-action certificates recomputed from the same counts under
alternative bound constructions.  Seeds 201 and 203 have no eligible gate and
are omitted.  Only the reported construction is confirmatory; the alternatives
are shown to locate the cause of failure.}
\label{tab:supp-bound-sensitivity}
\footnotesize
\begin{tabular}{@{}lrrrl@{}}
\toprule
Construction & Seed & \(\ucbLoss\) & \(\ucbAct\) & Decision \\
\midrule
Clopper--Pearson, Bonferroni \(\gamma=0.05\) & 207 & 0.167 & 0.162 & fails both \\
(reported) & 211 & 0.135 & 0.113 & passes \\
 & 223 & 0.038 & 0.168 & fails activation \\
\addlinespace[2pt]
Clopper--Pearson, \v{S}id\'ak \(\gamma=0.0513\) & 207 & 0.167 & 0.162 & fails both \\
 & 211 & 0.135 & 0.112 & passes \\
 & 223 & 0.038 & 0.168 & fails activation \\
\addlinespace[2pt]
Clopper--Pearson, no allocation \(\gamma=0.10\) & 207 & 0.157 & 0.152 & fails both \\
 & 211 & 0.125 & 0.103 & passes \\
 & 223 & 0.031 & 0.156 & fails activation \\
\addlinespace[2pt]
Wilson score limit, \(\gamma=0.05\) & 207 & 0.165 & 0.160 & fails both \\
 & 211 & 0.133 & 0.111 & passes \\
 & 223 & 0.036 & 0.166 & fails activation \\
\bottomrule
\end{tabular}
\end{table}

The pass/fail pattern is identical under the four examined constructions.
Removing the multiplicity allocation entirely moves the largest bound by
roughly one percentage point, and replacing the exact bound with a Wilson limit
moves it by less; neither is enough to rescue seed 207 or seed 223, whose failing
bounds exceed the target by 0.012--0.018.  Two seeds fail earlier still, at
threshold tuning, where no bound is involved.  These checks show that the
reported failures are not rescued by removing the allocation or by the Wilson
comparison; other valid interval procedures would require separate analysis.

\section{Endpoint Evidence Required for a Certificate}

Because an exact upper bound depends on the accepted count and not on the total
study size, the certification role must supply enough accepted units for
\emph{each} endpoint population.  \Cref{tab:supp-sample-size} gives the smallest
accepted count \(n\) for which \(U(e,n;0.05)\le\alpha\), which is the planning
quantity a prospective study needs once its target and expected incident count
are fixed.

\begin{table}[!htbp]
\centering
\caption{Smallest accepted count \(n\) satisfying \(U(e,n;0.05)\le\alpha\) for
\(e\) observed incidents, at the per-endpoint level used here
(\(\gamma=\delta/2=0.05\)).  A target cannot be met at any sample size once
\(e/n\) approaches \(\alpha\), so the table is a floor, not a schedule.}
\label{tab:supp-sample-size}
\small
\begin{tabular}{@{}lrrrr@{}}
\toprule
Target \(\alpha\) & \(e=0\) & \(e=1\) & \(e=2\) & \(e=3\) \\
\midrule
0.05 & 59 & 93 & 124 & 153 \\
0.10 & 29 & 46 & 61 & 76 \\
0.15 & 19 & 30 & 40 & 50 \\
0.20 & 14 & 22 & 30 & 37 \\
0.25 & 11 & 18 & 23 & 29 \\
\bottomrule
\end{tabular}
\end{table}

Two consequences shape the present study.  The evidence-loss and all-accepted
activation endpoints are comfortably powered: the eligible all-action gates
accept 123--169 positives and 133--182 images in total.  The clean-only
activation component is not: those gates accept 10--13 clean images against a
requirement of 19 even with zero incidents, which is why that component is
reported descriptively rather than certified.  KolektorSDD fails the same test
with 13 certification positives.  Prospective designs should therefore size the
positive and clean populations separately from the overall image count.

\section{Released Record Schema}

The archived records predate the terminology used in this article.  Field names
map to the article as follows: \texttt{deletion\_*} fields carry the
evidence-loss endpoint (\texttt{deletion\_risk}, \texttt{deletion\_errors},
\texttt{deletion\_denominator}, \texttt{deletion\_ucb} correspond to
\(\lossrisk\), \(\eLoss\), \(\nLoss\), \(\ucbLoss\)), and \texttt{invention\_*}
fields carry the excess-activation endpoint (\texttt{invention\_risk},
\texttt{invention\_errors}, \texttt{invention\_denominator},
\texttt{invention\_ucb} correspond to \(\actrisk\), \(\eAct\), \(\nAct\),
\(\ucbAct\)).  Additionally, \texttt{coverage} is conditional coverage before
the certificate decision, \texttt{deployable\_coverage} is pass-gated coverage,
and \texttt{action\_pool} is the all-action policy.  The record names are
retained for archive stability; the article's terminology is the one that
carries the intended meaning, and the words \emph{deletion} and
\emph{invention} are deliberately not used as claims in the text.

\section{Reproducibility Resources}

Analysis code is available at
\url{https://github.com/nbbllxx0/SafeRestore}.  Split manifests, per-image and
per-action records, trained-model summaries, exact certificate counts,
figure/table sources, and regression tests will be released in that repository.
Those records are sufficient to reconstruct the inventory and splits, refit
score models, evaluate fixed gates, and regenerate all aggregates and figures
without retraining the detector or restorer.

\end{document}